\documentclass[11pt,a4wide]{article}

\usepackage{arxiv}

\usepackage[T1]{fontenc}
\usepackage[utf8]{inputenc}
\usepackage{lmodern}
\usepackage{microtype}
\usepackage{amsmath,amssymb,amsfonts,mathtools}
\usepackage{bm}
\usepackage{graphicx}
\usepackage{subcaption}
\usepackage{booktabs}
\usepackage{siunitx}
\usepackage{xcolor}
\usepackage{listings}
\usepackage{enumitem}
\usepackage{listings}
\usepackage{xcolor}
\usepackage{tikz}
\usetikzlibrary{arrows.meta,decorations.markings}
\usetikzlibrary{decorations.pathmorphing,decorations.pathreplacing,decorations.shapes}
\usetikzlibrary{calc}

\definecolor{SegA}{RGB}{50,120,200}
\definecolor{SegB}{RGB}{220,120,40}
\definecolor{Common}{RGB}{90,90,90}

\tikzset{
  baseedge/.style={black, line width=0.9pt, line cap=round},
  subedgeA/.style={SegA, line width=2.0pt, line cap=round},
  subedgeB/.style={SegB, line width=2.0pt, line cap=round},
  nodebase/.style={circle, fill=black, draw=white, line width=0.4pt, inner sep=0pt, minimum size=6.5pt},
  nodeA/.style={circle, fill=SegA, draw=white, line width=0.4pt, inner sep=0pt, minimum size=6.5pt},
  nodeB/.style={circle, fill=SegB, draw=white, line width=0.4pt, inner sep=0pt, minimum size=6.5pt},
  ringA/.style={circle, fill=none, draw=SegA, line width=1.1pt, inner sep=0pt, minimum size=13pt},
  ringB/.style={circle, fill=none, draw=SegB, line width=1.1pt, inner sep=0pt, minimum size=13pt},
  conn/.style={dashed, black!35, line width=0.5pt},
  label/.style={font=\footnotesize},
  ilabel/.style={font=\small}
}

\newcommand{\basechain}{%
  \foreach \i in {1,...,9} {\coordinate (p\i) at ({\i-1},0);}
  \draw[baseedge] (p1) -- (p2) -- (p3) -- (p4) -- (p5) -- (p6) -- (p7) -- (p8) -- (p9);
  \foreach \i in {1,...,9} {
    \node[nodebase] at (p\i) {};
    \node[label, below=4pt] at (p\i) {$p_{\i}$};
  }
}

\usepackage{comment}

\usepackage[sort&compress]{natbib}

\usepackage{url}
\usepackage[hidelinks]{hyperref}
\usepackage{doi}

\usepackage{amsthm}
\newtheorem{theorem}{Theorem}
\newtheorem{remark}{Remark}

\usepackage[normalem]{ulem} 

\usepackage{chngcntr} 

\usepackage{placeins} 

\newcommand{\mycomment}[1]{\marginpar{\tiny #1}}

\hypersetup{
  colorlinks=true,
  linkcolor=blue!60!black,
  citecolor=blue!60!black,
  urlcolor=blue!60!black
}

\lstdefinestyle{bashstyle}{
  basicstyle=\ttfamily\small,
  columns=fullflexible,
  breaklines=true,
  frame=single,
  backgroundcolor=\color{black!3}
}
\newcommand{\R}{\mathbb{R}}
\newcommand{\norm}[1]{\left\lVert #1 \right\rVert}
\newcommand{\Atheta}{M_{\theta}^{-1}}
\newcommand{\diag}{\operatorname{diag}}

\newcommand{\maybeincludegraphics}[2][]{%
  \IfFileExists{#2}{%
    \includegraphics[#1]{#2}%
  }{%
    \fbox{\begin{minipage}{0.85\linewidth}\centering
      Missing figure file:\par\texttt{\detokenize{#2}}
    \end{minipage}}%
  }%
}

\title{Overlapping Schwarz Attention:\newline Hierarchical Attention via Domain Decomposition}

\author{
  \href{https://orcid.org/0000-0003-1015-8736}{\includegraphics[scale=0.06]{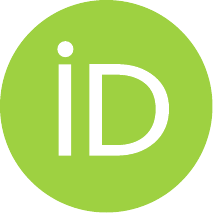}\hspace*{1mm}{Stephan Köhler}}
and
\href{https://orcid.org/0000-0002-9310-8533}{\includegraphics[scale=0.06]{orcid.pdf}\hspace*{1mm}{Oliver Rheinbach}}\\
Faculty of Mathematics and Computer Science\\ Technische Universität Bergakademie Freiberg, 09596 Freiberg, Germany
}

\date{\today}

\begin{document}

\maketitle

\begin{abstract}
We propose a hierarchical attention mechanism based on two-level overlapping Schwarz domain decomposition. The method is motivated by domain decomposition methods in partial differential equations which combine local subdomain corrections with a coarse level that communicates global, long-range information. We test its usefulness in the context of finite-dimensional operator learning using a simple, one-dimensional diffusion problem.
Although elementary, this problem provides a controlled sequence-to-sequence setting in which the exact nonlocal solution operator is known. After discretization, learning the solution operator amounts to approximating the inverse of a symmetric positive definite matrix. As a baseline, we use a global softmax-free low-rank attention operator of the form \(QK^T\). The proposed construction replaces this global factorization by a two-level additive structure: local low-rank attention blocks on overlapping subdomains are combined with a coarse attention block. The resulting operator has the form
$$
  M_{\theta}^{-1}
  =
  \Phi Q_0 K_0^T \Phi^T
  +
  \sum_{i=1}^{N}
  R_i^T D_i^{1/2} Q_i K_i^T D_i^{1/2} R_i .
$$
Here, $R_i$ restricts to an overlapping subdomain, $D_i$ is a partition-of-unity weight, and $\Phi$ is a coarse interpolation 
matrix.
Numerical experiments for synthetic Fourier right-hand sides indicate that the domain-decomposition attention operator can converge faster and can give more accurate approximations than a global low-rank attention baseline while using significantly fewer parameters.
\end{abstract}

\let\thefootnote\relax\footnotetext{
Two large language models were used independently to generate two separate code prototypes, which were cross-checked against each other. Large language models were also used for postprocessing and manuscript preparation. All generated code and text were reviewed, tested, and verified by the authors, who take full responsibility for the contents of this paper.
}

\section{Introduction}
The Transformer architecture, with its self-attention mechanism 
\citep{Vaswani2017},
has been central to the success of modern large language models.
Unlike earlier recurrent neural networks (RNNs), attention enables the direct and parallel modeling of contextual dependencies across long sequences. Formally, attention can be understood as a family of learnable interaction operators whose efficiency 
arises from the exploitation of low-rank factorizations in the computation of token interactions.

A standard self-attention layer starts from an input matrix
\[
  X \in \mathbb{R}^{m\times d_{\rm in}},
\]
where the rows correspond to $m$ tokens, grid points, or degrees of
freedom, and the $d_{\rm in}$ columns contain the associated features. The query, key,
and value matrices are obtained by learned linear mappings,
\begin{equation}
  Q = XW_Q,\qquad
  K = XW_K,\qquad
  V = XW_V,
\label{eq:QKV}
\end{equation}
where the entries of
  $W_Q,W_K\in\mathbb{R}^{d_{\rm in}\times d_k},\,
  W_V\in\mathbb{R}^{d_{\rm in}\times d_v}$
are the learnable parameters.

The standard, scaled dot-product attention operator with softmax is then
$$
\operatorname{Att}_{\rm softmax}(X)
=
\operatorname{softmax}\!\left(
  \frac{QK^T}{\sqrt{d_k}}
  \right)V,
$$
where \(d_k\) is the number of columns of $W_Q$ and $W_K$\textcolor{black}{, and {\rm softmax} is applied row-wise~\citep{Vaswani2017}}.
The matrix \(QK^T\in\mathbb{R}^{m\times m}\) contains pairwise interactions between the $m$ tokens or, in our case, (interior) grid points.
The parameter $d_k$ corresponds to the rank $r$ of the low-rank factorization $QK^T$, i.e.,
\begin{equation}
r={\rm rank}(QK^T)\le d_k.
\label{eq:rank}
\end{equation}
%
%
%
In this work, we will omit the softmax operator.
We also omit the $\sqrt{d_k}$ scaling,
which controls the magnitude of the dot products before applying softmax.
%
%
Thus, we have the linear attention operator
$$
  \operatorname{Att}(X)
  =
  \left(
    {QK^T}
  \right)V =
  \left({XW_Q W_K^TX^T}\right)XW_V.
$$
%
%
%
%
%
Our setting
allows us to focus on whether 
methods from domain decomposition for partial differential equations \citep{ToselliWidlund2005,SmithBjorstadGropp1996,Lions1988} can
be used to design more efficient attention operators.

%
%


Our model problem is a discretized one-dimensional Poisson problem.
After discretization, the solution map is the inverse of the sparse system matrix. 
A global low-rank attention layer can represent a rank-$d_k$ approximation of this inverse.
However, elliptic inverses have multiscale structure: local coupling is most important, but long-range low-frequency components must also be represented.
This is the structure exploited, e.g., by two-level Schwarz domain decomposition methods 
for partial differential equations.

In our two-level attention mechanism the fine level consists of local low-rank attention blocks on overlapping subdomains.
The coarse level consists of an attention block acting on a small interface-hat coarse basis.
The construction is not meant as a replacement for classical solvers.
Rather, it is a controlled experiment to demonstrate that domain de\-compo\-sition concepts can provide a useful structure, e.g., for operator learning methods or, potentially, natural language processing (NLP).
Well-known works in operator learning include the FEONet by~\citet{feonet2025}, the DeepONet~by \citet{Lu2021DeepONet} and the
Fourier Neural Operator (FNO)
by \citet{Li2021FNO}.
Modern overlapping Schwarz methods are highly scalable parallel methods
for solving
partial differential equations~\citep{jolivet2012scalable,heinlein2022parallel,DeparisForti2016,cai:NKS}, and may help to parallelize attention similarly to other known approaches in NLP~\citep{rectified2025,star2025,moba2025,quest2024}.



%

We note that hierarchical forms of attention have been studied previously in different settings~\citep{hierarchical2025,hierarchical2023,hierarchical2016}. These approaches are not directly comparable to the one proposed here, since both the underlying mechanisms and the intended applications are different. In the present work, the hierarchy is induced by an overlapping domain decomposition and is used to structure a softmax-free attention approximation of elliptic solution operators.
Hierarchical methods in NLP include the works by
\citet{hierarchicalnlp2024},
\citet{hierarchicalnlp2022},
\citet{hierarchicalnlp2021}, and
\citet{hierarchicalnlp2016}.



After completion of a first draft of this manuscript, we became aware of the closely related work by~\citet{choose2021}, which first
used
linear attention for operator learning. Our approach takes a further step in two respects. First, by an appropriate choice of the input matrix $X$ and the projection matrices $W_Q$, $W_K$, and $W_V$, the trainable parameters are identified directly with the factors $Q$ and $K$, so that the learned operator takes the simple 
form $QK^T$, rather than appearing only through the composite expression $XW_QW_K^TX^T$. Second, building on this interpretation, we introduce Schwarz attention, a new hierarchical attention architecture derived from a two-level overlapping Schwarz domain decomposition. It combines local low-rank attention operators with a global coarse correction and thereby transfers the central local-to-global mechanism of domain decomposition to attention-based operator learning.
In the next section, we specify the choice of $X$, $W_Q$, $W_K$, and $W_V$ that leads to this direct parametrization of the learned operator.

\section{One-Dimensional Diffusion Model Problem}
\label{sec:poisson}
We consider the Poisson problem in one dimension
\begin{equation}
  -u''(x) = f(x), \qquad x \in (0,1)\subset\mathbb{R}^1,
  \qquad u(0)=u(1)=0 .
  \label{eq:poisson-continuous}
\end{equation}
Let $n$ denote the number of uniform line elements on $[0,1]$.
Then $n+1$ is the number of grid points, and $n-1$ is the number of interior grid points. Let
\begin{equation}
  x_j = \frac{j}{n}, \qquad j=1,\ldots,(n-1) .
\end{equation}
Using the standard finite difference 
gives
\begin{equation}
  A u = f,
  \quad
  A = \frac{1}{h^2}
  \begin{pmatrix}
    2 & -1 \\
    -1 & 2 & -1 \\
       & \ddots & \ddots & \ddots \\
       &        & -1 & 2 & -1 \\
       &        &    & -1 & 2
  \end{pmatrix}
   \in\mathbb{R}^{(n-1)\times(n-1)},
  \qquad
  h=\frac{1}{n} .
  \label{eq:poisson-discrete}
\end{equation}
The exact discrete solution operator is therefore
$
  f \mapsto u=A^{-1}f .
$

The learning task in this paper is to approximate \(A^{-1}\) by a trainable structured operator $\Atheta$,  where $\theta$ denotes the parameters.
%
%
\textcolor{black}{For our one-dimensional model problem, we represent
$(n-1)$ active degrees of freedom
by a one-hot encoding, rather than by their physical coordinates, and augment this representation by the right-hand side $f$; i.e., our feature matrix is}
$$
X= [I_{n-1}, f]=\left(
\begin{array}{cccc|c}
1 & 0 & \ldots     &  0  & f_1\\
0 &  \ddots &  \ddots &  \vdots & f_2 \\
\vdots & \ddots & \ddots & 0 & \vdots \\
0  &   \ldots     & 0 &  1 & f_{n-1}\\
\end{array}
\right)
\in \mathbb{R}^{(n-1)\times n},
$$
where $I_{n-1}$ is the identity matrix of size $n-1$.
Thus, $d_{\rm in}=n$, i.e., the first $n-1$ columns identify the (interior) grid points, while the last column
contains the values of the right-hand side $f$.

Next, the weight matrix $W_V$ is chosen as a fixed matrix which extracts the right-hand side channel from $X$, 
\[
  V = XW_V = f \in \mathbb{R}^{n-1},
  \qquad
  W_V^T =
  (
    0,
    \ldots
    0,
    1
  )
  \in \mathbf{R}^{n}
  .
\]
As a result, in our context, $W_V$ does not contain learnable parameters but is a fixed matrix.


We then define $W_Q$ and $W_K$ by
$$
  W_Q =
  \begin{pmatrix}
    \Theta_Q \\
    0
  \end{pmatrix},
  \qquad
  W_K =
  \begin{pmatrix}
    \Theta_K \\
    0
  \end{pmatrix},
$$
where the matrices
$\Theta_Q$
and
$\Theta_K$
contain the learnable parameters in $W_Q$ and $W_K$.
Here, \(\Theta_Q,\Theta_K\in\mathbb{R}^{{n-1}\times d_k}\),
and the zeros in \(W_Q\) and \(W_K\) denote
one zero row. 
%
The identity block in $X$ resulting from the one-hot coding of the grid points then extracts $\Theta_Q$ and $\Theta_K$ from $W_{Q}$ and $W_{K}$, i.e.,
\[
  {Q}=XW_Q=\Theta_Q,
\qquad
  {K}=XW_K=\Theta_K. 
  \]
Now, the entries of $Q$ and $K^T$ are the learnable parameters.


As a result of our definitions, 
in the softmax-free linear setting considered in this paper,
the low rank factorization $QK^T$ is applied directly to the right-hand side, i.e.,
defining $u_\theta=\operatorname{Att}(X)$, we obtain
\begin{equation}
  u_\theta
  =
  \left(
    {QK^T}
  \right)V
  =
  QK^T f.
  \label{eq:utheta}
\end{equation}
The learning task is therefore to find the
attention factors such that
\begin{equation}
  QK^T \approx A^{-1},
\label{eq:learn}
\end{equation}
where \(A\) is the finite-difference or finite element Poisson matrix.

\begin{remark}

Although our numerical experiments are carried out for a one-dimensional elliptic model problem, the construction should not be viewed as being tied to this particular problem setting.
Rather, the model problem provides a controlled setting in which a nonlocal sequence-to-sequence operator is known exactly and can therefore be used to study the effect of hierarchical, domain-decomposition-induced attention.
The introduction of hierarchy in attention has been an active area of research and has allowed long contexts in LLMs; also see the discussion in Section~\ref{sec:discussion}.
Thus, our study may be
relevant more broadly for one-dimensional sequence models, where local interactions and long-range couplings have to be represented simultaneously.
%
The application of overlapping Schwarz attention to NLP is future
work. Recent work has investigated the decay of token influence with
distance in NLP~\citep{green2026}, helping to bridge the gap to settings in
partial differential equations where such decay is known.
\end{remark}




\section{Softmax-Free Attention as Low-Rank Approximate Inverse}

The baseline for our hierarchical attention is the global low-rank linear attention applied to operator learning, 
\begin{equation}
  M_{\theta,\mathrm{global}}^{-1} = QK^T,
  \qquad
  Q,K \in \R^{(n-1)\times r_g} .
  \label{eq:global-attention}
\end{equation}
Applied to a right-hand side $f$, the predicted solution is
\begin{equation}
  u_\theta = QK^T f;
\end{equation}
see (\ref{eq:utheta}).
The operator's rank is limited by the rank parameter denoted $r_g$.

\begin{remark}
Since $A^{-1}$ is symmetric positive definite, one could restrict to a symmetric factorization $BB^T$.
We deliberately do not impose this constraint here.
Both the global and the domain-decomposition models use independent factors $Q$ and $K$.
This gives the optimizer a less constrained parameterization while still allowing the learned operator to approximate a symmetric inverse.
\end{remark}
\begin{remark}
Although the approximate inverse is written as $QK^T$, the dense matrix \(QK^T\) is typically not formed explicitly. In all computations, the action on a right-hand side $f$ can be evaluated as
$
   Q(K^T f),
$
which costs \(O((n-1)r_g)\) for a single right-hand side instead of explicitly forming the
$(n-1)\times (n-1)$ matrix $(QK^T)$. 
\end{remark}

\section{Hierarchical Attention by Two-Level Overlapping Schwarz Domain Decomposition}

\subsection{Overlapping Subdomains and Coarse Hat Functions}
\label{sec:hat}

The subdomain construction used in our implementation is based on a decomposition of the edges of the one-dimensional chain graph. Each edge corresponds to a line element connecting two neighboring degrees of freedom. Thus, the partition is element-based rather than degree-of-freedom-based. Consequently, neighboring nonoverlapping subdomains share an interface degree of freedom before overlap is added.
This follows the standard finite element viewpoint.
%
%
Figure~\ref{fig:nonoverlapping} shows the nonoverlapping partition of the line elements resulting in two index sets $\mathcal{I}_1$ and $\mathcal{I}_2$ defining two nonoverlapping subdomains and the
interface.

In standard overlapping Schwarz methods, often generous overlap is introduced to reduce boundary effects.
In all our numerical experiments, however, we use a relatively small overlap of \(n_\delta=2\) since this seemed to be sufficient.
Hence, two neighboring overlapping subdomains have $2n_\delta+1=5$ common grid points;
Figure~\ref{fig:overlaptwo} illustrates the corresponding overlapping index sets for overlap two.

Let $\mathcal I_i$ be the index set associated with the $i$-th overlapping subdomain
and $n_i=|\mathcal I_i|$. Then $R_i\in\mathbb R^{n_i\times (n-1)} $ denotes the Boolean restriction matrix that extracts the degrees of freedom in \(\mathcal I_i\) from a global vector in \(\mathbb R^{n-1}\). Equivalently, \(R_i u\) is the subvector of \(u\) associated with the overlapping subdomain \(\mathcal I_i\), while \(R_i^T\) extends a local vector by zero to the global index set.

{\color{black}
The coarse space used in our experiments is spanned by interface-centered hat functions. More precisely, the coarse degrees of freedom are placed at the interface nodes of the nonoverlapping partition, and the corresponding hat functions attain their maxima at these nodes while decreasing linearly in the adjacent subdomains.

We also investigated piecewise constant partition-of-unity basis functions. Such basis functions perform a uniform averaging over each subdomain and are closely related to mean pooling, one of the simplest token aggregation mechanisms in Transformer models~\citep{meanpooling}. They are also conceptually related to the coarse token representations employed in hierarchical Transformer architectures such as MoBA and FCA~\citep{moba2025,zhao2022fine}.
In contrast, the interface-centered hat basis performs a distance-weighted
averaging. 
Since these alternative coarse bases did not outperform the interface-hat basis in our experiments, they are not considered further.
}

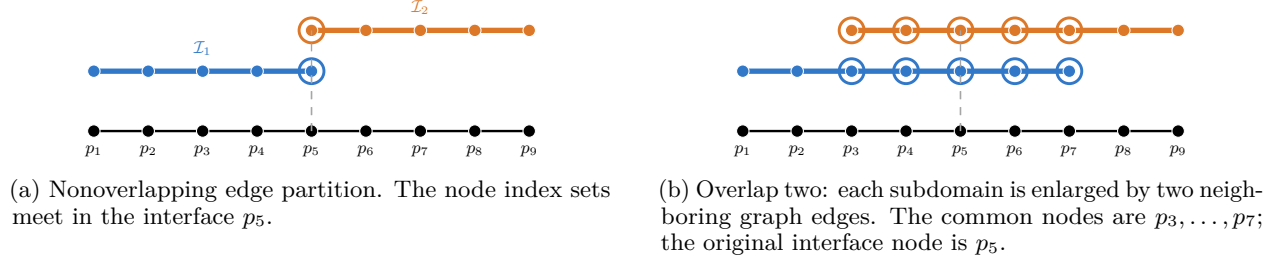
\begin{figure}[t]
\centering

\begin{subfigure}[t]{0.48\textwidth}
\centering
\begin{tikzpicture}[line join=round, scale=0.72, transform shape]
  \basechain

  \foreach \i in {1,...,5} {\coordinate (a\i) at ({\i-1},1.10);}
  \foreach \i in {5,...,9} {\coordinate (b\i) at ({\i-1},1.85);}

  \draw[subedgeA] (a1) -- (a2) -- (a3) -- (a4) -- (a5);
  \draw[subedgeB] (b5) -- (b6) -- (b7) -- (b8) -- (b9);

  \foreach \i in {1,...,5} {\node[nodeA] at (a\i) {};}
  \foreach \i in {5,...,9} {\node[nodeB] at (b\i) {};}

  \node[ringA] at (a5) {};
  \node[ringB] at (b5) {};
  \draw[conn] (p5) -- (a5);
  \draw[conn] (p5) -- (b5);

  \node[ilabel, SegA, above=5pt] at ($(a1)!0.5!(a5)$) {$\mathcal{I}_{1}$};
  \node[ilabel, SegB, above=5pt] at ($(b5)!0.5!(b9)$) {$\mathcal{I}_{2}$};
\end{tikzpicture}
\caption{Nonoverlapping edge partition. The node index sets meet in the interface $p_5$.}
\label{fig:nonoverlapping}
\end{subfigure}
\hfill
\begin{subfigure}[t]{0.48\textwidth}
\centering
\begin{tikzpicture}[line join=round, scale=0.72, transform shape]
  \basechain

  \foreach \i in {1,...,7} {\coordinate (a\i) at ({\i-1},1.10);}
  \foreach \i in {3,...,9} {\coordinate (b\i) at ({\i-1},1.85);}

  \draw[subedgeA] (a1) -- (a2) -- (a3) -- (a4) -- (a5) -- (a6) -- (a7);
  \draw[subedgeB] (b3) -- (b4) -- (b5) -- (b6) -- (b7) -- (b8) -- (b9);

  \foreach \i in {1,...,7} {\node[nodeA] at (a\i) {};}
  \foreach \i in {3,...,9} {\node[nodeB] at (b\i) {};}

  \foreach \i in {3,4,5,6,7} {
    \node[ringA] at (a\i) {};
    \node[ringB] at (b\i) {};
  }

  \draw[conn] (p5) -- (a5);
  \draw[conn] (p5) -- (b5);

\end{tikzpicture}
\caption{Overlap two: each subdomain is enlarged by two neighboring graph edges. The common nodes are $p_3,\ldots,p_7$; the original interface node is $p_5$.}
\label{fig:overlaptwo}
\end{subfigure}

\caption{Nonoverlapping and overlapping node index sets for two subdomains.}
\label{fig:subdomain-overlap}
\end{figure}

\subsection{Partition of Unity in the Overlap}

In the overlap, several subdomains may cover the same degree of freedom.
Let \(m_j\) be the number of overlapping subdomains containing grid point \(j\).
On subdomain \(i\), define the diagonal weight matrix
\begin{equation}
  D_i = \diag\left({1}/{m_j}: j\in\mathcal I_i\right) .
\end{equation}
In the decompositions used in the experiments below, each grid point in an overlap region is covered by two neighboring subdomains, so that $m_j=2$ there, while $m_j=1$ outside the overlap regions.

Then the weights form a discrete partition of unity:
$
  \sum_{i=1}^{N} R_i^T D_i R_i = I_{n-1} .
$
In the experiments below we use the symmetric overlap weighting
\begin{equation}
  R_i^T D_i^{1/2} G_i D_i^{1/2} R_i .
\end{equation}

\subsection{Local Attention Blocks}

On each of the $N$ overlapping subdomains, we use a local low-rank attention operator
\begin{equation}
  G_i = Q_iK_i^T,
  \qquad
  Q_i,K_i\in\R^{n_i\times r_\ell},
\end{equation}
where $r_\ell$ denotes the rank of the local attention blocks.
%
The corresponding fine-level operator is
\begin{equation}
  M_{\theta,\mathrm{loc}}^{-1}
  =
  \sum_{i=1}^{N}
  R_i^T D_i^{1/2} Q_iK_i^T D_i^{1/2} R_i .
  \label{eq:local-schwarz-attention}
\end{equation}
This operator is sparse in the 
sense that
each learned block acts only on an overlapping subdomain.

\subsection{Coarse Attention Block}

In Poisson problems, local blocks alone cannot represent global low-frequency components.
We therefore add a coarse attention space.
Let
\begin{equation}
  \Phi \in \R^{(n-1)\times n_0}
\end{equation}
be a coarse interpolation matrix, i.e., $\Phi$ interpolates coarse coefficients to the fine grid.
Its transpose $\Phi^T$ is used as the corresponding restriction to the coarse space.

In the experiments below, \(\Phi\) consists of interface hat functions{\color{black}; see Figure~\ref{fig:operator-matrices} in the lower rightmost panel}.
For \(N\) subdomains this gives \(n_0=N-1\) coarse basis functions,
one associated with each pair of neighboring disjoint index sets.

We denote the rank of the coarse attention by $r_0$.
The coarse attention operator is
\begin{equation}
  G_0 = Q_0K_0^T,
  \qquad
  Q_0,K_0\in\R^{n_0\times r_0} .
\end{equation}
Interpolated to the fine grid, the coarse contribution is
\begin{equation}
  M_{\theta,\mathrm{coarse}}^{-1} = \Phi Q_0K_0^T\Phi^T .
  \label{eq:coarse-attention}
\end{equation}

\subsection{Two-Level Additive Hierarchical Attention Mechanism}

Combining the fine and coarse levels gives the proposed hierarchical attention operator
\begin{equation}
  M_{\theta,\mathrm{Schwarz}}^{-1}
  =
  \Phi Q_0K_0^T\Phi^T
  +
  \sum_{i=1}^{N}
  R_i^T D_i^{1/2} Q_iK_i^T D_i^{1/2} R_i.
  \label{eq:dd-attention}
\end{equation}
This is the central construction of the paper.
It can be interpreted as a two-level additive Schwarz-inspired attention layer.
The term ``hierarchical'' refers to the coexistence of local fine-scale attention blocks and a global coarse attention mechanism.
We refer to the proposed architecture as two-level overlapping Schwarz attention, or Schwarz attention for short.


\section{Training Procedure}
\subsection{Loss Function}

For a batch of right-hand sides \(f^{(b)}\), the exact solutions are computed as
$
u^{(b)} = A^{-1} f^{(b)} .
$
The model prediction is
$
u_\theta^{(b)} = \Atheta f^{(b)} .
$
A plain mean-squared error (MSE) tends to underweight examples whose true solutions have small amplitude.
For Poisson problems, high-frequency right-hand sides produce solutions whose amplitudes are damped by the inverse Laplacian.
We therefore use a sample-wise weighted MSE (wMSE)
%
%
\begin{equation}
  \mathcal L_{\mathrm{wMSE}}(\theta)
  =
  \frac{1}{B}
  \sum_{b=1}^{B}
  \left((n-1)^{-1}\norm{u_\theta^{(b)}-u^{(b)}}_2^2\right)/
       \left(\max\left\{(n-1)^{-1}\norm{u^{(b)}}_2^2,\varepsilon\right\}\right).
  \label{eq:weighted-mse}
\end{equation}
where $B$ denotes the batch size and $u_\theta^{(b)}$ and $u^{(b)}$
denote the predicted and exact solutions for sample $b$, respectively.
All our experiments use this loss, setting $\varepsilon=1e-30$.
Without this weighting, good approximations are not achieved for high-frequency right-hand sides. 
Unless stated otherwise, all experiments use a batch size of $B=256$.
\subsection{Initialization of the Attention Factors}
\label{sec:init-fair-comparison}
\label{sec:init-random}

{
For the random initialization we use a resolution- and rank-aware scaling of the factor matrices.
The entries of the trainable factors $Q$ and $K$ are first drawn randomly. An entry of $QK^T$ is an inner product of length $r$ of a row of $Q$ with a row of $K$, see~\eqref{eq:rank}.

We then multiply both factors by
$
\left({h}/{4}\right)^{1/2} r^{-1/4},
h=\frac1n.
$
Thus the product $QK^T$ is scaled by
$
\left({h}/{4}\right)r^{-1/2}.
$
The factor $r^{-1/2}$ compensates the growth of the random inner products with the rank, while the factor $h$ matches the natural size of the inverse one-dimensional Laplacian, whose largest entries are of order $h/4$.

}

\subsection{Synthetic Right-Hand Sides}

The training data are generated on the fly; no finite stored training set
is used.  We sample right-hand sides from a mixed Fourier family and
normalize each generated right-hand side to unit Euclidean norm.  The
experiments use the first $16$ sine modes and
also the first $16$ cosine modes on the interior grid points.
The mixed Fourier generator combines two types of right-hand sides in
each batch.  One half of the batch consists of pure Fourier modes: a
single normalized sine or cosine mode is selected and multiplied by a
random sign.  These pure modes are not damped by a decay factor.  The
other half consists of random Fourier combinations.  In this part, sine
and cosine modes $m=1,\ldots,16$ are combined, with higher modes damped by
the factor $m^{-1.5}$.  After concatenating the two parts, the batch is
shuffled and every right-hand side is normalized.

The global attention model and the hierarchical
domain-decomposition attention model see the same sequence of training
right-hand sides.  
%
%
We use $16$ evaluation right-hand sides in the solution plots; they are generated once with a fixed test seed and are then used for all models.
%
%
For each generated right-hand side, the reference
solution is computed by applying the exact discrete inverse,
$$
  u^{(b)} = A^{-1} f^{(b)}.
$$
Section~\ref{app:rhs} of the Appendix
defines the generator used for our mixed Fourier right-hand sides.

%


\begin{remark}
The distribution of right-hand sides is an essential part of our operator-learning
problem. A low-rank attention operator cannot be expected to approximate the
inverse Poisson operator uniformly well for all possible right-hand sides if the
chosen ranks are much smaller than the problem dimension.
%
The mixed Fourier distribution used here is therefore a controlled test class.
The goal is not to learn a uniformly accurate inverse on the complete space $\mathbb{R}^{n-1}$, but to compare global low-rank
attention and Schwarz attention on a given 
class of right-hand sides with both local and global structure.
\end{remark}

\section{First Experimental Setup}
\label{sec:exp}
Our first numerical experiments consist of learning-rate sweeps for a
problem size of $n=256$ line elements, corresponding to $255$ active degrees of freedom,
and \(N=8\) subdomains;
we always use an overlap of $n_\delta=2$, and the local ranks are set to $r_\ell=4$, the coarse rank is set to $r_0=N-1=7$. The coarse basis is the interface-hat basis; see Section~\ref{sec:hat}.

For \(n=256\), \(N=8\), and overlap two, the nonoverlapping
partition consists of eight subdomains with $32$ line elements each. After
adding two overlap elements on each interior side, the six interior
overlapping subdomains contain 36 elements, corresponding to 37 active
degrees of freedom. The two boundary subdomains contain 34 elements each;
after eliminating the Dirichlet boundary values, they contain 34 active
degrees of freedom.
\textcolor{black}{Hence $\sum_{i=1}^{8} n_i = 2\cdot 34 + 6\cdot 37 = 290$. With local rank \(r_\ell=4\), the local factors therefore contain $2r_\ell \sum_{i=1}^{8} n_i = 2\cdot 4\cdot 290 = 2320$ trainable parameters. The coarse space has \(N-1=7\) interface hat functions. Using rank \(r_0=7\) for the coarse attention block adds $2(N-1)r_0 = 2\cdot 7\cdot 7 = 98$ parameters. Thus, the total number of trainable parameters for the overlapping Schwarz attention is \(2418\).}

We use AdamW~\citep{loshchilov2019decoupled} with zero weight decay, $\beta_1=0.9$,
$\beta_2=0.999$, and $\varepsilon=10^{-8}$ and batch size $256$.
Since the weight decay is zero, the resulting updates are equivalent
to those of Adam.
We use a \texttt{ReduceLROnPlateau} scheduler with reduction factor $0.5$
and patience 200.

\subsection{Learning-Rate Sweep for the Overlapping Schwarz Attention}
\begin{figure}[h!]
  \centering
  \begin{subfigure}{0.42\linewidth}
    \centering
    \includegraphics[width=\linewidth]{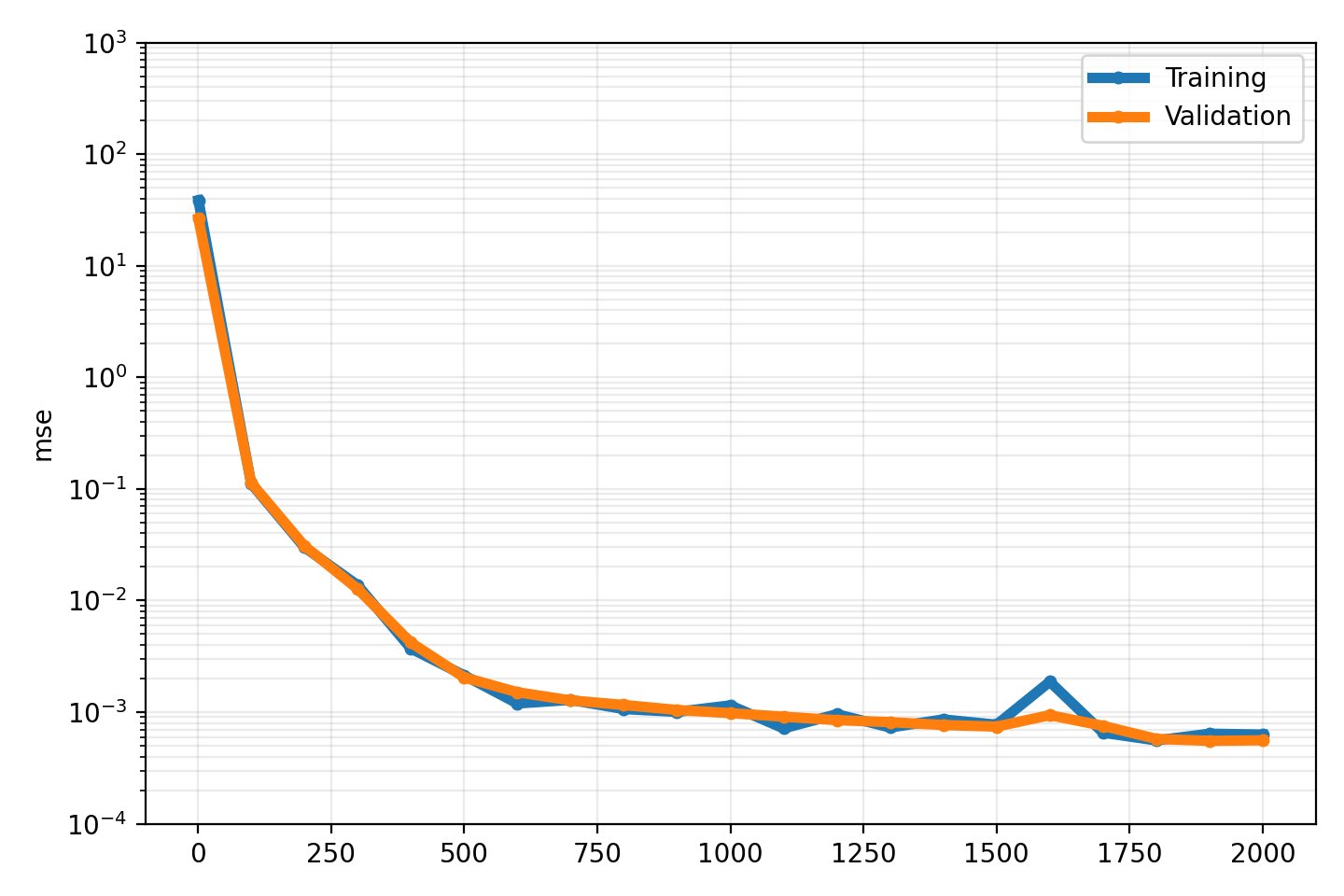}
    \caption{$\eta=10^{-3}$}
  \end{subfigure}
  \hfill
  \begin{subfigure}{0.42\linewidth}
    \centering
    \includegraphics[width=\linewidth]{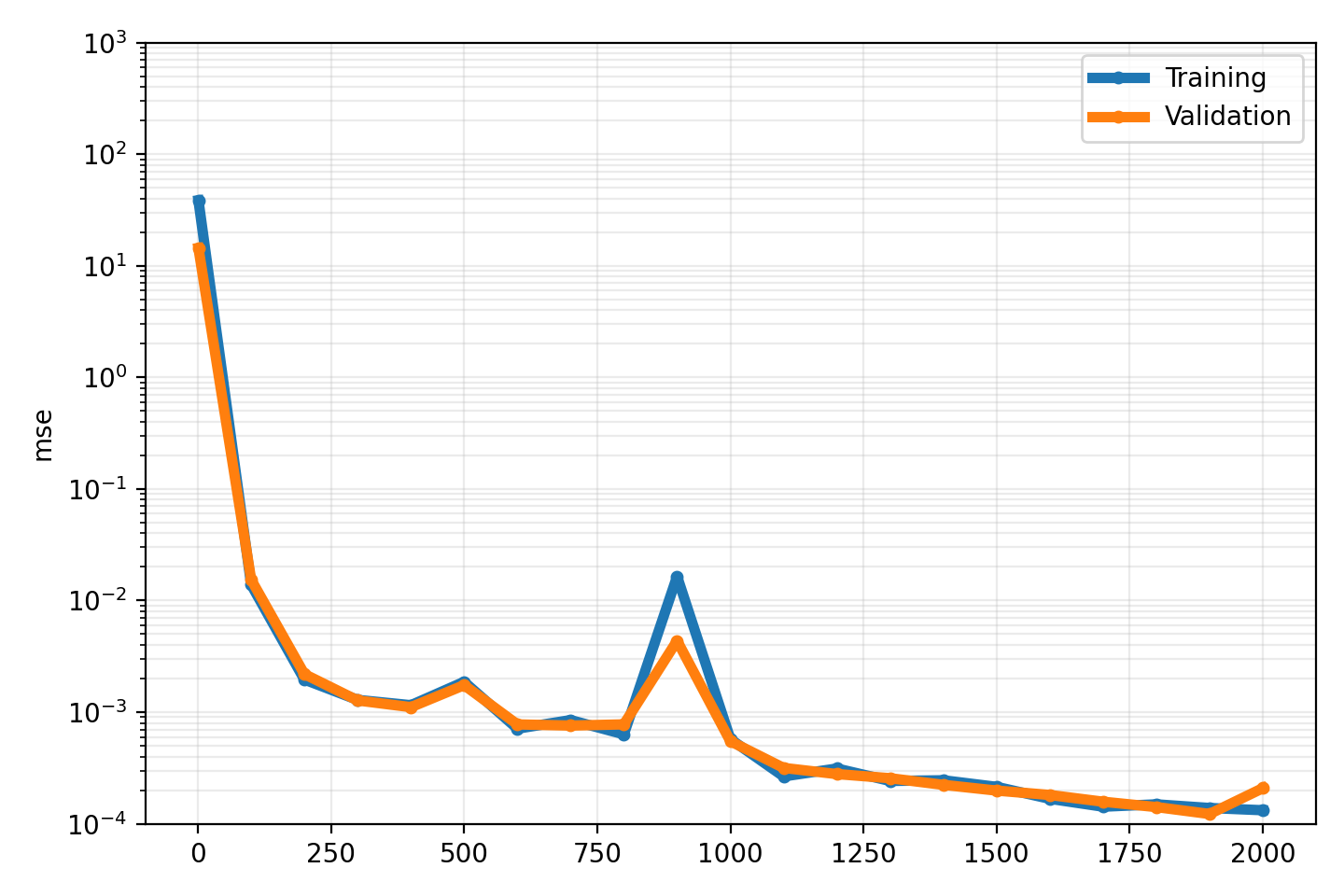}
    \caption{$\eta=3\cdot 10^{-3}$}
  \end{subfigure}
  \caption{
  Representative learning-rate histories for the two-level overlapping Schwarz
  attention.  The conservative choice $\eta=10^{-3}$ reaches the same low-error
  regime as the more aggressive choice $\eta=3\cdot 10^{-3}$, while avoiding
  the strongest transient fluctuations.
  The full learning-rate sweep for Schwarz attention is reported in
Section~\ref{app:lrsweep} of the Appendix. The comparison with global attention over the complete sweep is given in Section~\ref{app:lrsweepcomp}.
  }
  \label{fig:lr-representative}
\end{figure}

Figure~\ref{fig:lr-representative} shows two representative
histories. 
The more aggressive value $\eta=3\cdot 10^{-3}$ attains one
of the lowest losses in the sweep, but exhibits a visible 
spike. We decided to use $\eta=10^{-3}$ as a conservative default in the following experiments.
The full learning-rate sweep
    is reported in Section~\ref{app:lrsweep} of the Appendix; the comparison with global attention over
the complete sweep is given in Section~\ref{app:lrsweepcomp}.






\subsection{Comparison of Schwarz Attention and Global Attention}

\begin{table}[t]
  \centering
  \begin{tabular}{llrrrr}
    \toprule
    Learning  & Model & Final & Final  & $\|M^{-1}_{\theta}\!-\!A^{-1}\|_F$ & $\displaystyle
\min_{\operatorname{rank}(B)\le r}
\lVert B-A^{-1}\rVert_F$ \\
    rate \(\eta\) &  & train\,wMSE &  val.\,wMSE &  &  \\

    \midrule
    $10^{-4}$ & Schwarz attention & $2.587\cdot 10^{-2}$ & $2.612\cdot 10^{-2}$ & $0.0712$ & $2.488\cdot 10^{-4}$ \\
    $10^{-4}$ & Global attention, rank 5 & $6.149\cdot 10^{-1}$ & $6.086\cdot 10^{-1}$ & $0.263$ & $4.504\cdot 10^{-3}$ \\
    $10^{-4}$ & Global attention, rank 39 & $5.466\cdot 10^{-2}$ & $5.403\cdot 10^{-2}$ & $0.244$ & $2.488\cdot 10^{-4}$ \\[.3ex]
    $3\cdot 10^{-4}$ & Schwarz attention & $1.408\cdot 10^{-3}$ & $1.270\cdot 10^{-3}$ & $0.0697$ & $2.488\cdot 10^{-4}$ \\
    $3\cdot 10^{-4}$ & Global attention, rank 5 & $3.150\cdot 10^{-1}$ & $3.151\cdot 10^{-1}$ & $0.294$ & $4.504\cdot 10^{-3}$ \\
    $3\cdot 10^{-4}$ & Global attention, rank 39 & $2.564\cdot 10^{-3}$ & $2.549\cdot 10^{-3}$ & $0.226$ & $2.488\cdot 10^{-4}$ \\[.3ex]
    $10^{-3}$ & Schwarz attention & $6.293\cdot 10^{-4}$ & $5.594\cdot 10^{-4}$ & $0.0585$ & $2.488\cdot 10^{-4}$ \\
    $10^{-3}$ & Global attention, rank 5 & $2.419\cdot 10^{-1}$ & $2.429\cdot 10^{-1}$ & $0.193$ & $4.504\cdot 10^{-3}$ \\
    $10^{-3}$ & Global attention, rank 39 & $8.630\cdot 10^{-4}$ & $8.527\cdot 10^{-4}$ & $0.154$ & $2.488\cdot 10^{-4}$ \\[.3ex]
    $3\cdot 10^{-3}$ & Schwarz attention & $1.321\cdot 10^{-4}$ & $2.082\cdot 10^{-4}$ & $0.0474$ & $2.488\cdot 10^{-4}$ \\
    $3\cdot 10^{-3}$ & Global attention, rank 5 & $2.477\cdot 10^{-1}$ & $2.326\cdot 10^{-1}$ & $0.11$ & $4.504\cdot 10^{-3}$ \\
    $3\cdot 10^{-3}$ & Global attention, rank 39 & $3.102\cdot 10^{-4}$ & $3.189\cdot 10^{-4}$ & $0.116$ & $2.488\cdot 10^{-4}$ \\[.3ex]
    $10^{-2}$ & Schwarz attention & $1.604\cdot 10^{-4}$ & $2.258\cdot 10^{-4}$ & $0.0462$ & $2.488\cdot 10^{-4}$ \\
    $10^{-2}$ & Global attention, rank 5 & $2.391\cdot 10^{-1}$ & $2.361\cdot 10^{-1}$ & $0.0722$ & $4.504\cdot 10^{-3}$ \\
    $10^{-2}$ & Global attention, rank 39 & $3.097\cdot 10^{-2}$ & $2.980\cdot 10^{-2}$ & $0.265$ & $2.488\cdot 10^{-4}$ \\[.3ex]
    $3\cdot 10^{-2}$ & Schwarz attention & $6.487\cdot 10^{-4}$ & $5.899\cdot 10^{-4}$ & $0.0733$ & $2.488\cdot 10^{-4}$ \\
    $3\cdot 10^{-2}$ & Global attention, rank 5 & $2.489\cdot 10^{-1}$ & $2.548\cdot 10^{-1}$ & $0.144$ & $4.504\cdot 10^{-3}$ \\
    $3\cdot 10^{-2}$ & Global attention, rank 39 & $6.592\cdot 10^{-1}$ & $6.567\cdot 10^{-1}$ & $0.868$ & $2.488\cdot 10^{-4}$ \\
    \bottomrule
  \end{tabular}
 \caption{Learning-rate sweep summary for $n=256$ and $N=8$.
Each row reports the last logged training wMSE, validation wMSE, and
absolute operator Frobenius error after 2000 training steps. The final
column gives the Frobenius error of the best rank-$r$ approximation of
$A^{-1}$.
Here, $r=5$ for the global rank-$5$ model and $r=39$ for the global
rank-$39$ model; for Schwarz attention,
$r=R_{\mathrm{total}}=39$ is the corresponding rank bound.
}
  \label{tab:lr-comp-long}
\end{table}

  \begin{figure}[t]
    \centering
    \includegraphics[width=.48\linewidth]{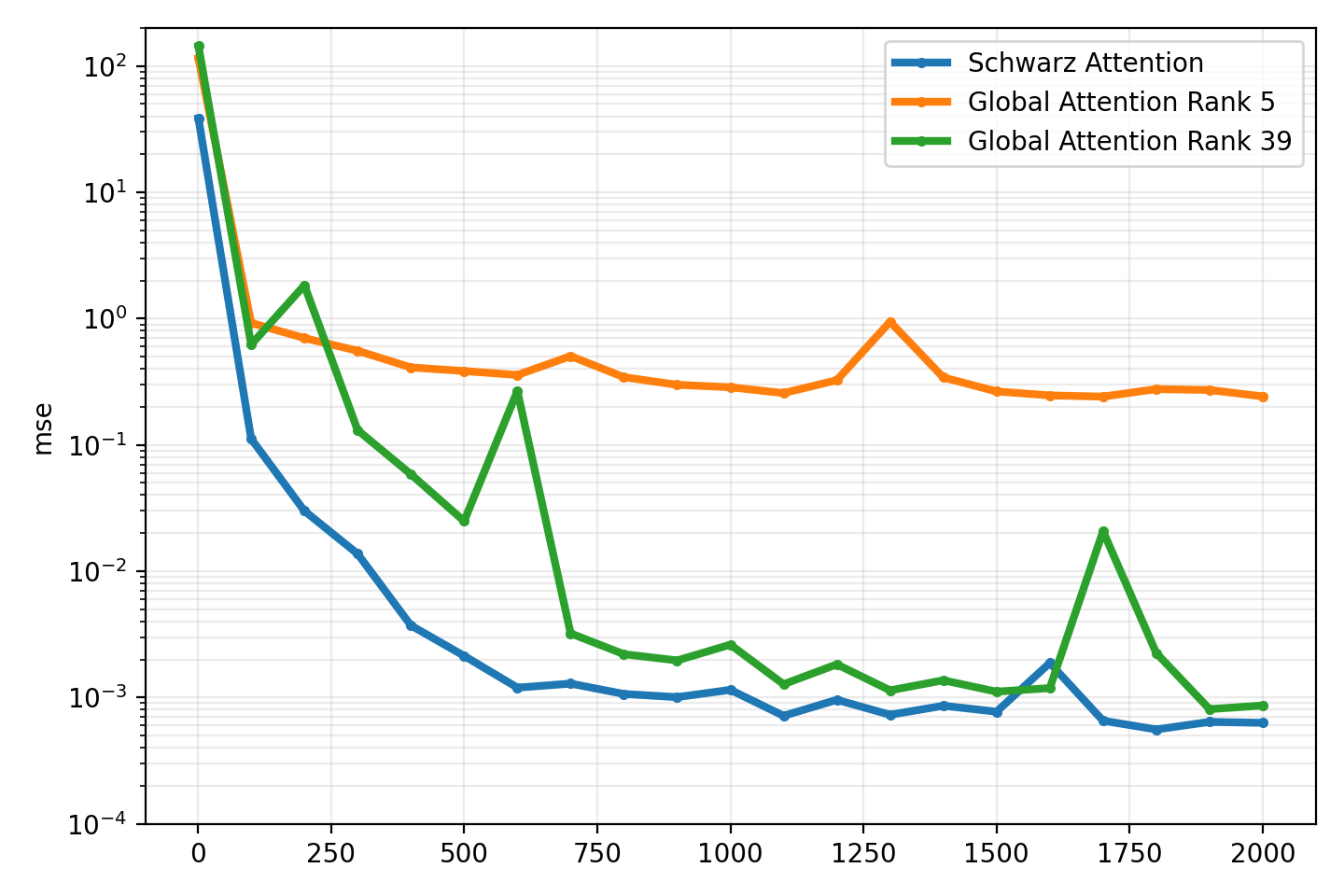}
    \caption{Comparison of the Schwarz attention with
    the
    standard global attention for $n=256$, $N=8$, and $\eta=10^{-3}.$
    The full set of training curves is shown in Section~\ref{app:lrsweepcomp} of the Appendix.}
    \label{fig:comp_1e-3}
  \end{figure}

We now perform a comparison of the Schwarz attention with the standard global attention.
The rank of the global attention operator, which is our baseline, is set to $r_g=39.$
This is a fair comparison with respect to the rank of the operator since the rank of the Schwarz attention operator is bounded by $Nr_\ell+r_0=8\times 4+7=39$.
However, the Schwarz attention has only $2418$ trainable parameters; see~Section~\ref{sec:exp}.
%
\textcolor{black}{The corresponding global low-rank attention with rank \(39\) has $2\cdot 255\cdot 39 = 19890$ trainable parameters, since the two Dirichlet boundary values are eliminated and the global operator acts on \(255\) active degrees of freedom, i.e., more than eight times as many parameters as the Schwarz attention model. A parameter-comparable global baseline is obtained by choosing rank \(5\) for the global attention, which
gives $2\cdot 255\cdot 5 = 2550$ parameters. This is close to the \(2418\) parameters of the Schwarz attention model and is therefore the parameter-matched global baseline.}



Figure~\ref{fig:comp_1e-3}
compares the Schwarz attention model with our two global low-rank baselines
using our default learning rate $\eta=10^{-3}.$
%
The Schwarz attention reaches a lower error level than both global low-rank baselines.  The global rank-5 baseline is roughly comparable in parameter count but stagnates at a much higher loss, whereas the substantially larger global rank-39 baseline still does not match the Schwarz model. The complete learning-rate sweep is given in
Section~\ref{app:lrsweepcomp} of the Appendix;
it shows that the Schwarz attention consistently outperforms the global attention.

Table~\ref{tab:lr-comp-long} summarizes the final values after 2000 optimization steps for all learning rates in the sweep. In addition to the final training and validation wMSE, we report the 
operator error
$
{|M_\theta^{-1}-A^{-1}|_F}.
$
This quantity measures the approximation quality of the learned inverse operator itself. 
The final column reports the Frobenius error of the best rank-$r$
approximation of $A^{-1}$. Since $A^{-1}$ is symmetric positive
definite, this approximation can equivalently be obtained by truncating
its spectral decomposition.

The results show that the ordering observed in the wMSE curves is also reflected at the operator level.
The parameter-matched global rank-5 model remains far from the operator baseline and yields substantially larger Frobenius errors. Increasing the global rank to 39 improves the global model considerably, but the Schwarz attention model still attains smaller operator errors over the relevant range of learning rates. This indicates that the advantage of the Schwarz construction is not only visible on sampled right-hand sides, but also in the learned operator approximation.
%
%
The best values for Schwarz attention are obtained in the moderate range around $10^{-3}$ to $10^{-2}$. 


\begin{remark}
We emphasize that the Frobenius error of the full operator is
reported in Table~\ref{tab:lr-comp-long} only as an additional diagnostic quantity.
It is different from the training objective, which minimizes a sample-wise weighted forward error.
In particular, the learned operators are not computed as best rank-$r$
approximations of $A^{-1}$ in the Frobenius norm.
Hence a comparatively large Frobenius error does not contradict small relative solution errors on the right-hand sides used for evaluation.
\end{remark}
These results indicate that the domain-decomposition structure is highly effective for learning the inverse Poisson operator in our context. The combination of local low-rank attention blocks with a coarse attention space yields both higher approximation accuracy and substantially improved parameter efficiency.


For the comparison in the next sections, we use the conservative learning rate $\eta=10^{-3}$.

\subsection{Approximation Quality on Fixed Visualization Examples}


\begin{figure}[t]
\begin{subfigure}{0.325\textwidth}
\includegraphics[width=.999\linewidth]{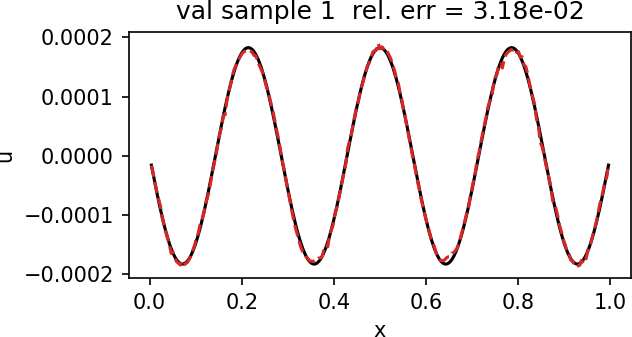}
\caption{Overlapping Schwarz attention.}
\end{subfigure}
\begin{subfigure}{0.325\textwidth}
\includegraphics[width=.999\linewidth]{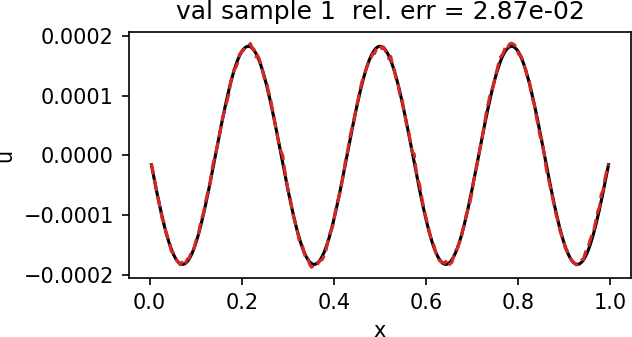}
\caption{Global attention with rank $r=39$.}
\end{subfigure}
\begin{subfigure}{0.325\textwidth}
\includegraphics[width=.999\linewidth]{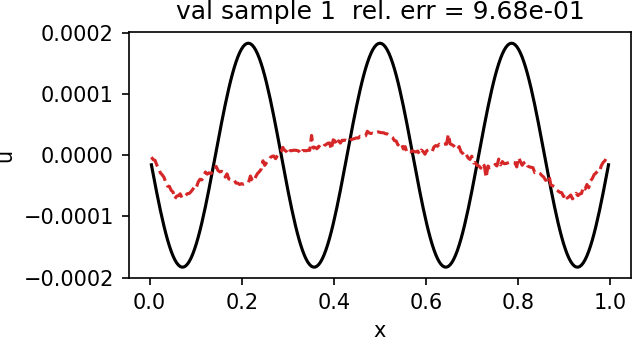}
\caption{Global attention with rank $r=5$.}
\end{subfigure}
\caption{
Sample solutions for (a) Schwarz attention, (b) global attention with rank $39$, and (c) global attention with rank $5$ after
2000 iterations with learning rate $\eta=0.001$. In each panel, the exact solution is plotted in black and the model prediction as a dashed red line.
Additional sample solutions are shown in Section~\ref{app:visual} of the Appendix.
}
\label{fig:visualcrop}
\end{figure}

Figure~\ref{fig:visualcrop} shows an example of the learned approximations for the overlapping Schwarz attention, global attention with rank 39, and global attention with rank 5.
While the Schwarz attention and the rank-39 global attention perform well, the parameter-matched rank-5 global attention shows poor approximation quality in this example.

\section{Increasing the Number of Local Attention Blocks}
\begin{figure}[ht]
  \centering
  \includegraphics[width=.89\linewidth]{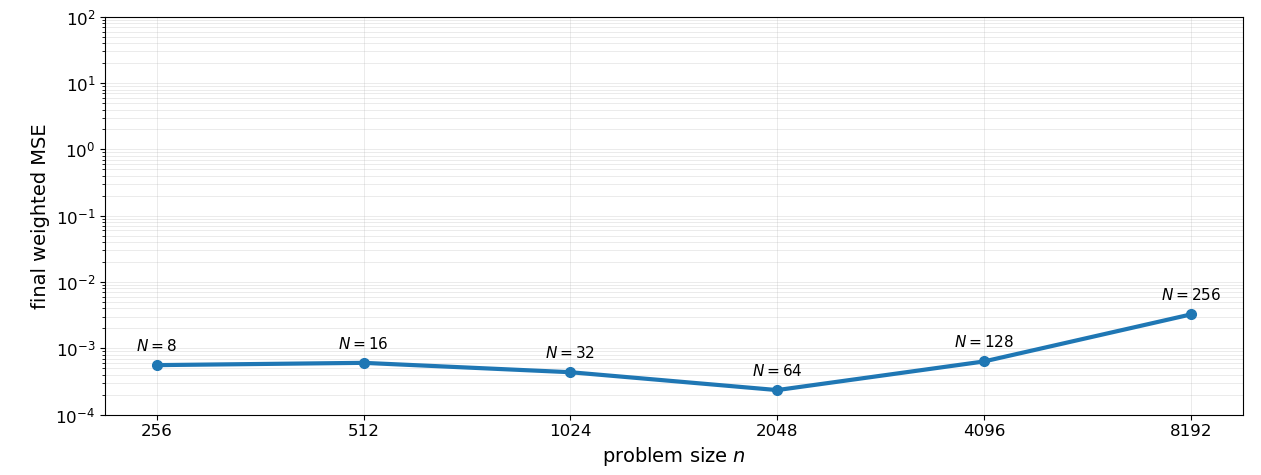}
\caption{Final weighted MSE in the subdomain-scaling experiment.
The problem size $n$ and the number of subdomains $N$ are increased
proportionally, while the local rank is kept fixed at $r_\ell=4$.
The error remains in the range $10^{-4}$--$10^{-3}$ up to $n=4096$
and increases for the largest case $n=8192$.
}
\label{fig:figweak_}
\end{figure}
%
%
%
%
%
%
%
\begin{table}[t]
\centering
\begin{tabular}{r r r r r r r r }
\toprule
 n & N & \#Params & Final\,val.\,wMSE & $\|M^{-1}_{\theta}\!-\!A^{-1}\|_F$ & $\displaystyle\min_{\operatorname{rank}(B)\le r} \| B-A^{-1}\|_F$ & Rank bound $R_{\rm total}$\\

 \midrule
    256 & 8 & 2418 & $5.594\cdot 10^{-4}$ & $5.846\cdot 10^{-2}$ & $2.488\cdot 10^{-4}$ & 39 \\
    512 & 16 & 5138 & $6.054\cdot 10^{-4}$ & $4.414\cdot 10^{-2}$ & $8.719\cdot 10^{-5}$ & 79 \\
    1024 & 32 & 11346 & $4.373\cdot 10^{-4}$ & $5.906\cdot 10^{-2}$ & $3.069\cdot 10^{-5}$ & 159 \\
    2048 & 64 & 26834 & $2.348\cdot 10^{-4}$ & $7.984\cdot 10^{-2}$ & $1.083\cdot 10^{-5}$ & 319 \\
    4096 & 128 & 70098 & $6.347\cdot 10^{-4}$ & $1.363\cdot 10^{-1}$ & $3.824\cdot 10^{-6}$ & 639\\ 
    8192 & 256 & 205778 & $3.248\cdot 10^{-3}$ & $2.194\cdot 10^{-1}$ & $1.351\cdot 10^{-6}$ & 1279\\ 
    \bottomrule
\end{tabular}
\caption{Results for increasing problem size
$n=256,\ldots,8192$ and a proportionally increasing number of local
attention blocks. Each row reports the final weighted MSE, the absolute
operator Frobenius error, and the Frobenius error of the best rank-$r$
approximation of $A^{-1}$, where
$r=R_{\mathrm{total}}=Nr_\ell+r_0$ is the rank bound of the Schwarz
attention operator.}
\label{tab:weak}
\end{table}

We perform additional experiments, inspired by parallel weak scaling in supercomputing, i.e., we increase the problem size $n=256, 512, 1024, 2048, 4096,$ and  $8192$ while also proportionally increasing the number of local attention blocks $N=8,16,32,64,128,$ and $256$.
%
%
Thus, the size of each nonoverlapping subdomain remains essentially constant. The number of trainable parameters grows approximately linearly with the number of subdomains, apart from the contribution of the coarse block.

We use local ranks $r_{\ell}=4$ for the local attention blocks, and a coarse rank $r_0=N-1$ for the coarse attention block. 
Again, we use AdamW combined with a \texttt{ReduceLROnPlateau} scheduler with reduction factor $0.5$ and patience 200.
The initial learning rate for AdamW is $\eta=10^{-3}$.
%
%
%
%
%
We use mixed Fourier right-hand sides, a batch size of $256$, and
$2000$ training steps. The results are shown in Table~\ref{tab:weak} and Figure~\ref{fig:figweak_};
the corresponding training histories are shown in Figure~\ref{appfig:weak}.


The final error is not completely scale-independent: the largest case,
$n=8192$ and $N=256$, is visibly more difficult.  Nevertheless, the
optimization still reaches a final weighted MSE of order $10^{-3}$ while
using fixed local rank and hundreds of local attention blocks.

\section{Rank-Based Frobenius-Norm Bounds}
\label{subsec:rank-frobenius-benchmarks}


The exact discrete solution operator is \(A^{-1}\). Since \(A\) is
symmetric positive definite, we can write
\[
  A = V\Lambda V^T,
  \qquad
  A^{-1} = V\Lambda^{-1}V^T,
\]
where the eigenvalues of \(A\) satisfy
\[
  0<\lambda_1\leq \lambda_2\leq \cdots \leq \lambda_{n-1} .
\]
For the one-dimensional finite-difference Poisson matrix with homogeneous
Dirichlet boundary conditions,
\[
  \lambda_k
  =
  \frac{4}{h^2}
  \sin^2\left(\frac{k\pi}{2n}\right),
  \qquad k=1,\ldots,n-1 .
\]
Thus the eigenvalues of \(A^{-1}\) are \(\lambda_k^{-1}\). Since
\(A^{-1}\) is symmetric positive definite, these are also its singular
values, ordered decreasingly.

\begin{theorem}[Global low-rank attention]
\label{thm:global-rank-benchmark}
Let $M_{\theta,\mathrm{global}}^{-1}=QK^T$, with
$Q,K\in\R^{(n-1)\times r_g}$.

Then
$$
  \inf_{Q,K\in\R^{(n-1)\times r_g}}
  \|A^{-1}-QK^T\|_F^2
  =
  \sum_{k=r_g+1}^{n-1} \lambda_k^{-2}.
$$
Equivalently,
$$
  \frac{
    \inf_{Q,K\in\R^{(n-1)\times r_g}}\|A^{-1}-QK^T\|_F
  }{
    \|A^{-1}\|_F
  }
  =
  \left(
  \frac{\sum_{k=r_g+1}^{n-1} \lambda_k^{-2}}
       {\sum_{k=1}^{n-1} \lambda_k^{-2}}
  \right)^{1/2}.
$$
The infimum is attained by the truncated spectral decomposition
$$
  QK^T
  =
  \sum_{k=1}^{r_g} \lambda_k^{-1} v_kv_k^T .
$$
\label{theo:one}
\end{theorem}

\begin{theorem}[Rank bound for Schwarz attention]
\label{thm:schwarz-rank-benchmark}
Every $S\in\mathcal{S}_{\mathrm{Schwarz}}$ satisfies
$$
  \operatorname{rank}(S)
  \leq
  R_{\mathrm{total}}
  :=
  \min\{r_0,\operatorname{rank}(\Phi)\}
  +
  \sum_{i=1}^{N}\min\{r_i,n_i\}.
$$
Consequently,
$$
  \inf_{S\in\mathcal{S}_{\mathrm{Schwarz}}}
  \|A^{-1}-S\|_F
  \geq
  \left(
  \sum_{k=R_{\mathrm{total}}+1}^{n-1}
  \lambda_k^{-2}
  \right)^{1/2},
$$
with the convention that the sum is zero if
\(R_{\mathrm{total}}\geq (n-1)\). In relative form,
\[
  \frac{
    \inf_{S\in\mathcal{S}_{\mathrm{Schwarz}}}
    \|A^{-1}-S\|_F
  }{
    \|A^{-1}\|_F
  }
  \geq
  \left(
  \frac{
    \sum_{k=R_{\mathrm{total}}+1}^{n-1} \lambda_k^{-2}
  }{
    \sum_{k=1}^{n-1} \lambda_k^{-2}
  }
  \right)^{1/2}.
\]
\label{theo:two}
\end{theorem}

The second theorem is only a rank-based lower benchmark. It does not use
the structure of the Schwarz-attention
architecture. 
The proofs are given in Section~\ref{app:proofs} of the Appendix.

\section{Discussion}
\label{sec:discussion}

The experiments support three observations.
First, the proposed two-level domain-decomposition attention operator provides an effective hierarchical representation of the Poisson inverse.
The local blocks are designed to capture strong short-range interactions, while the coarse block represents weaker long-range interactions.
This mirrors the role of fine and coarse spaces in classical two-level Schwarz methods for elliptic partial differential equations.
Second, the proposed model is substantially more parameter-efficient than a global low-rank attention baseline.
In the representative configuration, the global baseline uses $19890$ trainable parameters, while the domain-decomposition model uses $2418$.

Despite using far fewer parameters,
the domain-decomposition model reaches low errors in fewer
optimization steps
%
%
and achieves slightly more accurate approximations in the reported experiments.
Third, the approach remained effective for up to $N=256$ local attention blocks, although a slight deterioration in the final wMSE was observed.
Overall, these results suggest that hierarchical attention derived from domain decomposition provides an effective inductive bias for operator learning.
Extending this framework 
to other contexts is future work.

There are several limitations.
The experiments are restricted to a one-dimensional Poisson problem and synthetic Fourier right-hand sides.
The model is linear and softmax-free.
This is intentional for the present study: it isolates the effect of domain-decomposition structure.
%
%
%
%

Extending the approach to standard softmax attention~\citep{Vaswani2017} is non-trivial, because the row-wise softmax operation generally requires access to the full attention scores.
A naive implementation would therefore require assembling the Schwarz attention scores into a dense matrix, or otherwise performing a global normalization across the local and coarse contributions, which could substantially increase computational and memory costs.
Nevertheless, core advantages of the hierarchical parameterization would likely persist: the model would still require significantly fewer trainable parameters than a global low-rank baseline, and the structured factorization may continue to facilitate faster and more stable optimization.

However, applying a single softmax normalization to a mixture of fine and coarse
tokens may be undesirable.
A possible direction for an adapted softmax variant would therefore be to apply softmax separately on the fine-level subdomain blocks and on the coarse level, treating the interface-hat functions as pooling or summary tokens.
Related ideas of combining attention mechanisms for long documents at different
granularities appear in Longformer \citep{longformer2020}.
Related ideas of processing representations at multiple spatial
granularities appear in the Swin Transformer
\citep{swintransformer2021}.

%
%
Extensions to softmax attention, as well as to other PDEs, higher-dimensional problems, nonlinear attention architectures and parallel aspects of our approach remain future work.
We note that the overlapping Schwarz algorithm is highly parallelizable~\citep{jolivet2012scalable,heinlein2022parallel}
and may therefore provide a route toward parallel attention mechanisms across multiple GPUs.
Such methods have received considerable attention in recent years for long-context Transformer models~\citep{rectified2025,star2025,moba2025,quest2024}.


This work does not involve human subjects or use personal data, nor does it target a safety-critical or socially sensitive application. We are not aware of specific societal harms arising directly from the proposed method in the setting studied here.


\section{Conclusion}

We introduced a hierarchical attention mechanism inspired by two-level overlapping Schwarz domain decomposition.
The method replaces a dense global low-rank attention operator by a sum of local overlapping attention blocks and a coarse attention block.
For the one-dimensional Poisson inverse, this construction gives a compact and interpretable operator-learning model.
The numerical results suggest that domain-decomposition structure can improve both parameter efficiency and training behavior in softmax-free operator learning,
relative to a global attention baseline.

\section{Interactive demonstration}
An interactive demonstration accompanying this paper is available at
\url{https://huggingface.co/spaces/rhbch/schwarzattention}.
It allows interactive exploration of the proposed Overlapping Schwarz
Attention and comparison with a parameter-matched global low-rank
attention baseline.
\newpage
\bibliographystyle{unsrtnat}
\bibliography{./airefs}


\newpage
\appendix
\counterwithin{figure}{section}
\counterwithin{table}{section}

\section{Supplementary Material}
\subsection{Construction of the Right-Hand Sides}
\label{app:rhs}

The code in Figure~\ref{fig:appendixrhs} shows how the right-hand sides are constructed.

\begin{figure}[h!]
\begin{lstlisting}[language=Python]
def normalize_rows(v: torch.Tensor, eps: float = 1.0e-12) -> torch.Tensor:
    return v / (v.norm(dim=-1, keepdim=True) + eps)

def fourier_basis(x: torch.Tensor, modes: int) -> torch.Tensor:
    """Rows are normalized sine/cosine Fourier modes without decay."""
    rows = []
    for m in range(1, modes + 1):
        rows.append(torch.sin(math.pi * m * x))
        rows.append(torch.cos(math.pi * m * x))
    return normalize_rows(torch.stack(rows, dim=0))

def sample_mixed_fourier_rhs(batch_size: int, x: torch.Tensor) -> torch.Tensor:
    """Mixed Fourier RHS: half pure modes, half random decaying combinations."""
    n_dof = x.numel()

    def pure_modes(bs: int) -> torch.Tensor:
        B = fourier_basis(x, MODES)
        idx = torch.randint(0, B.shape[0], (bs,), device=DEVICE)
        signs = 2.0 * torch.randint(0, 2, (bs, 1), device=DEVICE, dtype=DTYPE) - 1.0
        return signs * B[idx, :]

    def random_combinations(bs: int) -> torch.Tensor:
        f = torch.zeros(bs, n_dof, device=DEVICE, dtype=DTYPE)
        for m in range(1, MODES + 1):
            scale = 1.0 / (m ** DECAY)
            a = torch.randn(bs, 1, device=DEVICE, dtype=DTYPE) * scale
            b = torch.randn(bs, 1, device=DEVICE, dtype=DTYPE) * scale
            f = f + a * torch.sin(math.pi * m * x)[None, :]
            f = f + b * torch.cos(math.pi * m * x)[None, :]
        return f

    b1 = batch_size // 2
    b2 = batch_size - b1
    f = torch.cat([pure_modes(b1), random_combinations(b2)], dim=0)
    f = f[torch.randperm(batch_size, device=DEVICE)]
    return normalize_rows(f)
\end{lstlisting}
\caption{
The code was extracted from our more general implementation and subsequently verified to produce identical results.
}
\label{fig:appendixrhs}
\end{figure}

\subsection{Full Learning Rate Sweep for the Overlapping Schwarz Attention}
\label{app:lrsweep}


\begin{figure}[h!]
  \centering
  \begin{subfigure}{0.40\textwidth}
    \centering
    \includegraphics[width=\linewidth]{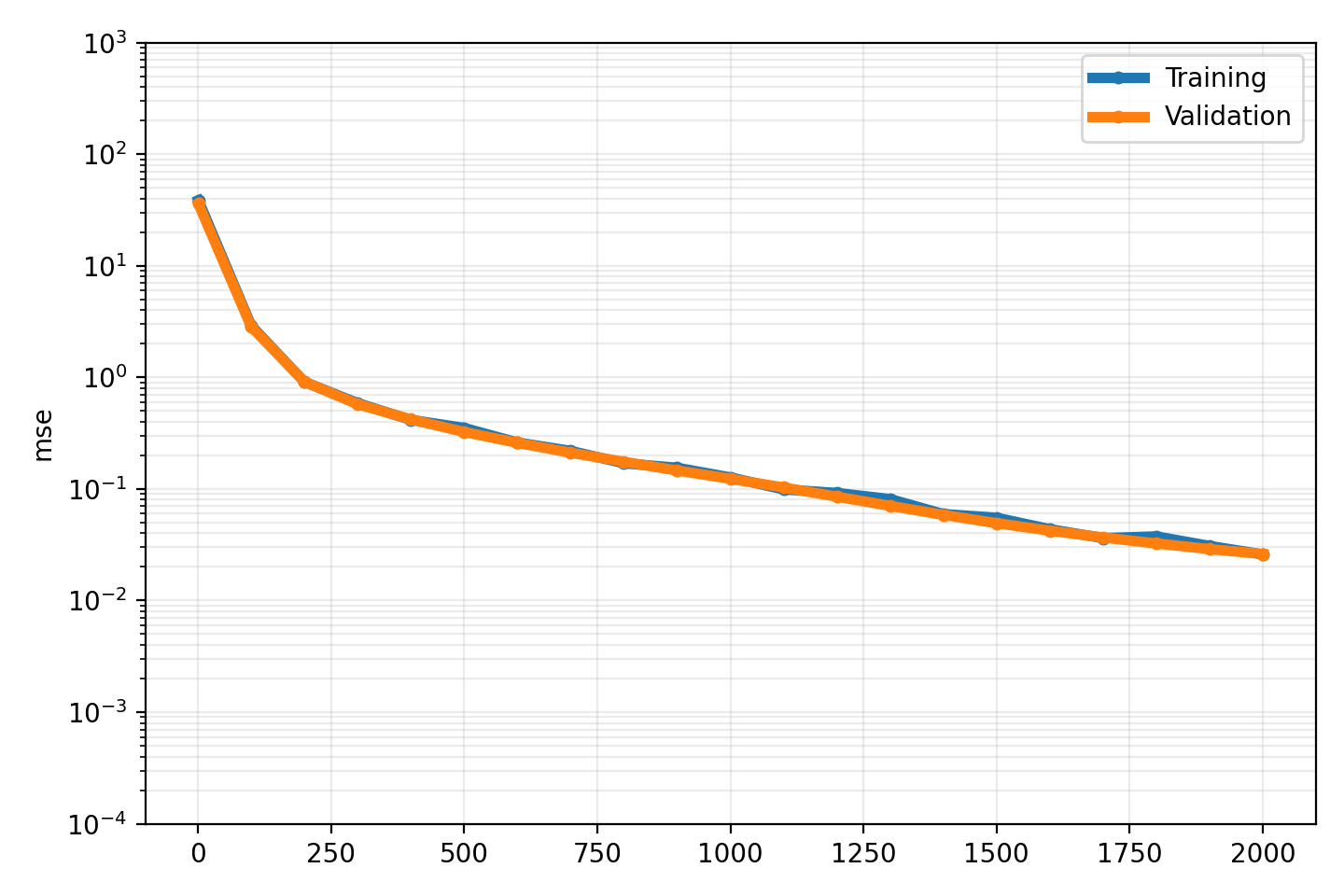}
    \caption{\(\eta=10^{-4}\)}
  \end{subfigure}
  \hfill
  \begin{subfigure}{0.40\textwidth}
    \centering
    \includegraphics[width=\linewidth]{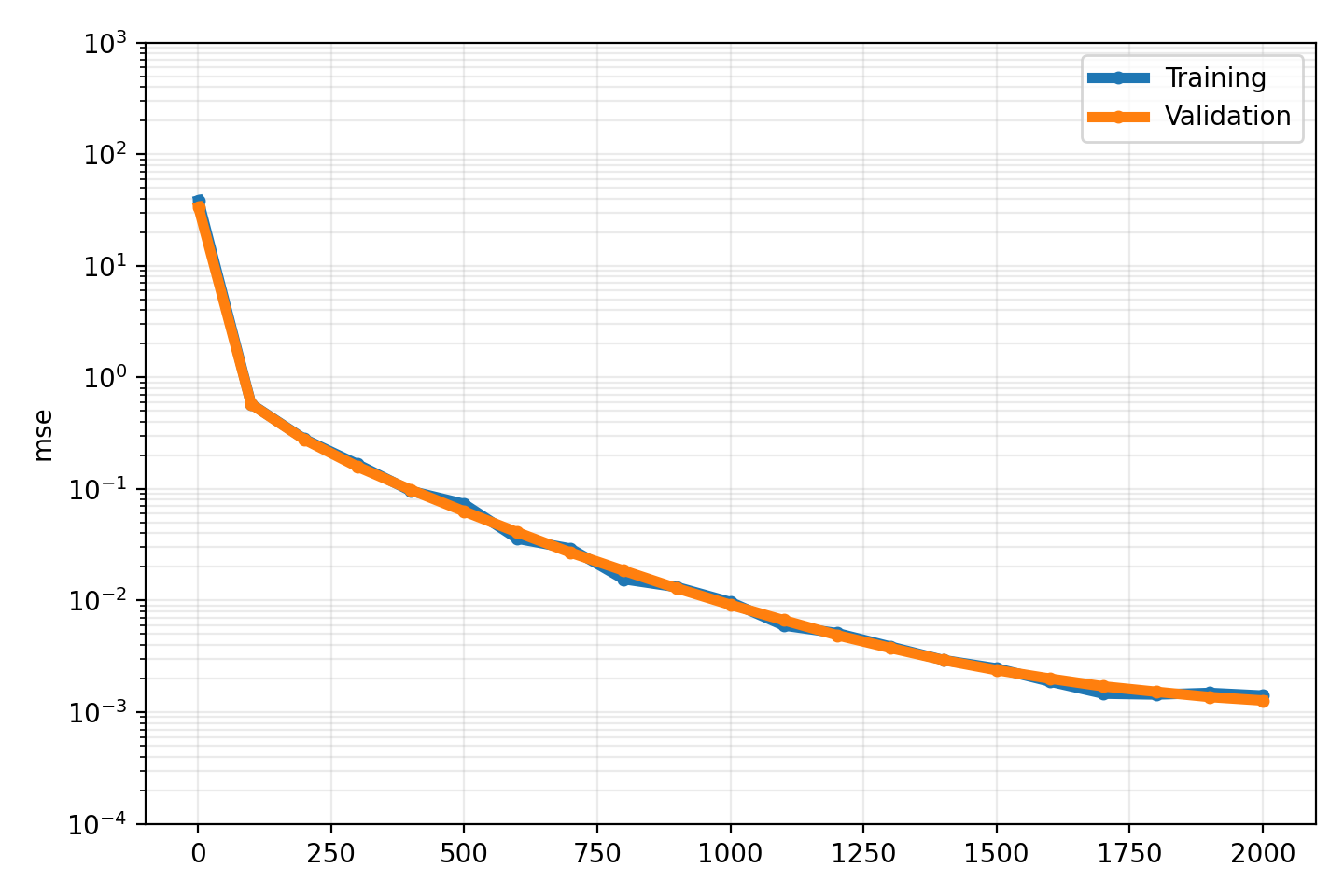}
    \caption{\(\eta=3\cdot10^{-4}\)}
  \end{subfigure}
  \vspace{0.5em}
  \begin{subfigure}{0.40\textwidth}
    \centering
    \includegraphics[width=\linewidth]{figures/schwarz_train_val_mse_by_lr_korrigiert/schwarz_train_val_mse__lr_0p001.png}
    \caption{\(\eta=10^{-3}\)}
  \end{subfigure}
  \hfill
  \begin{subfigure}{0.40\textwidth}
    \centering
    \includegraphics[width=\linewidth]{figures/schwarz_train_val_mse_by_lr_korrigiert/schwarz_train_val_mse__lr_0p003.png}
    \caption{\(\eta=3\cdot10^{-3}\)}
  \end{subfigure}
  \vspace{0.5em}
  \begin{subfigure}{0.40\textwidth}
    \centering
    \includegraphics[width=\linewidth]{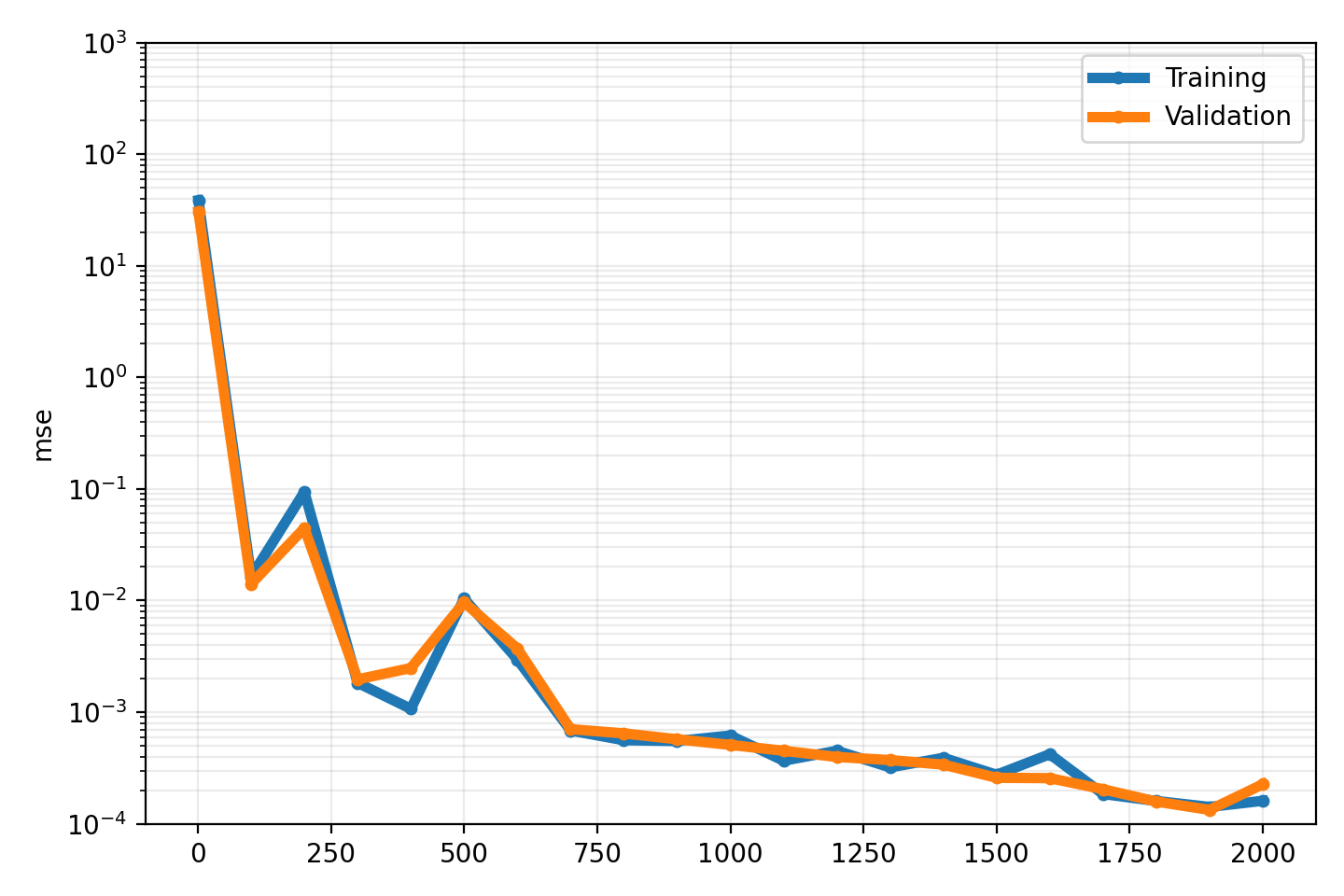}
    \caption{\(\eta=10^{-2}\)}
  \end{subfigure}
  \hfill
  \begin{subfigure}{0.40\textwidth}
    \centering
    \includegraphics[width=\linewidth]{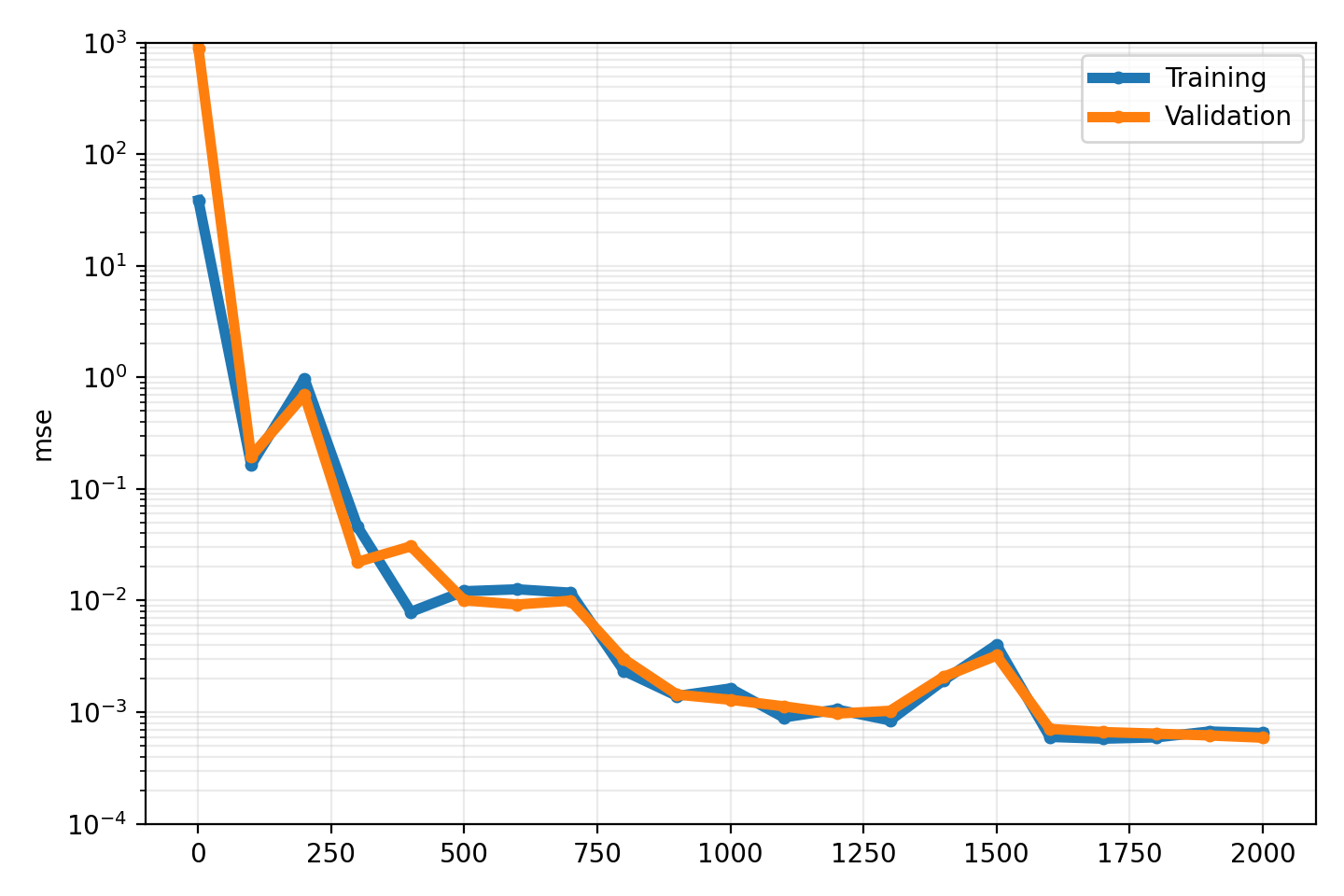}
    \caption{\(\eta=3\cdot10^{-2}\)}
  \end{subfigure}
%
%
%
  \caption{Learning-rate sweep for $n=256$ and $N=8$ for the two-level overlapping Schwarz attention using a rank $r_\ell=4$ for the local attention blocks and rank $r_0=7$ for the coarse attention. 
  Each panel shows the weighted training MSE and the weighted validation MSE.
  We use AdamW and a plateau scheduler.
  }
  \label{fig:lr-sweep-sk}
\end{figure}

Figure~\ref{fig:lr-sweep-sk} shows the training curves for the {six} learning rates
%
%
$$
\eta \in
\left\{
10^{-4},\,\,
3\cdot 10^{-4},\,\,
10^{-3},\,\,
3\cdot 10^{-3},\,\,
10^{-2},\,\,
3\cdot 10^{-2}
\right\}.
$$
For the initialization of the attention factors, we use the method described in~Section~\ref{sec:init-random}.

Figure~\ref{fig:lr-sweep-sk} shows that the training behavior is not highly sensitive to the precise choice of the learning rate over a fairly broad range. The smallest learning rates lead to stable but slower convergence, while the larger learning rates reach low errors more rapidly but show more pronounced transient fluctuations. In particular, $\eta=3\cdot 10^{-3}$ attains one of the lowest final losses in this sweep, but also exhibits a visible intermediate spike. We therefore use $\eta=10^{-3}$ as a conservative default in our experiments: it reaches essentially the same error regime, avoids the strongest transients, and provides a robust compromise between convergence speed and stability.

The validation wMSE is computed on 256 independently sampled right-hand sides drawn from the same distribution as the training data. 
The training and validation curves are almost indistinguishable throughout the sweep. Since the validation right-hand sides are independently sampled from the same distribution as the training data, this should be interpreted as in-distribution validation. The agreement indicates that the observed optimization behavior is not specific to the sampled training batch, but is representative for independently sampled right-hand sides from the same problem class.

\subsection{Full Learning-Rate Sweep: Schwarz Attention vs.\ Global Attention}

\label{app:lrsweepcomp}

 \begin{figure}[h!]
  \centering
  \begin{subfigure}{0.40\textwidth}
    \centering
    \includegraphics[width=\linewidth]{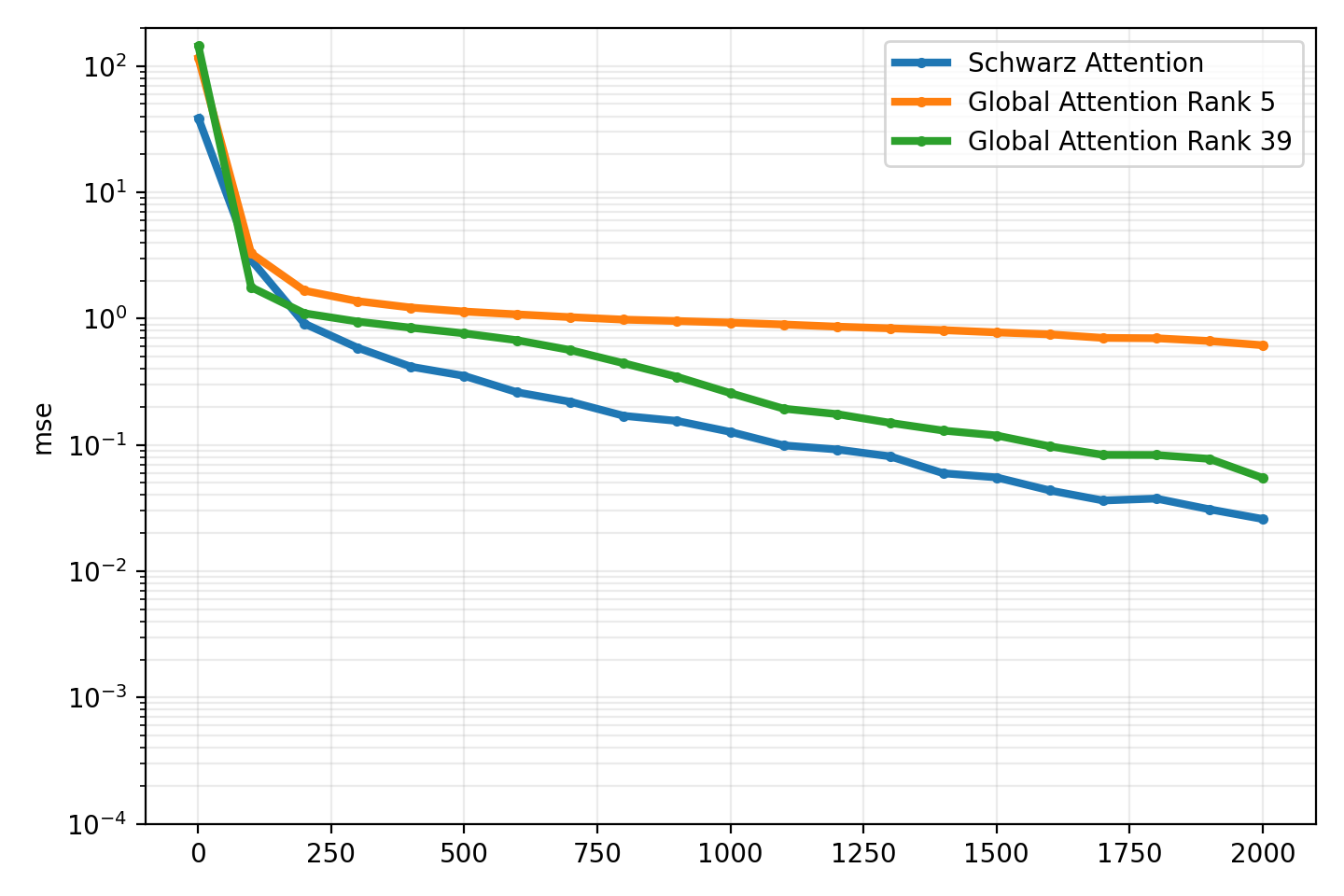}
    \caption{\(\eta=10^{-4}\)}
  \end{subfigure}
  \hfill
  \begin{subfigure}{0.40\textwidth}
    \centering
    \includegraphics[width=\linewidth]{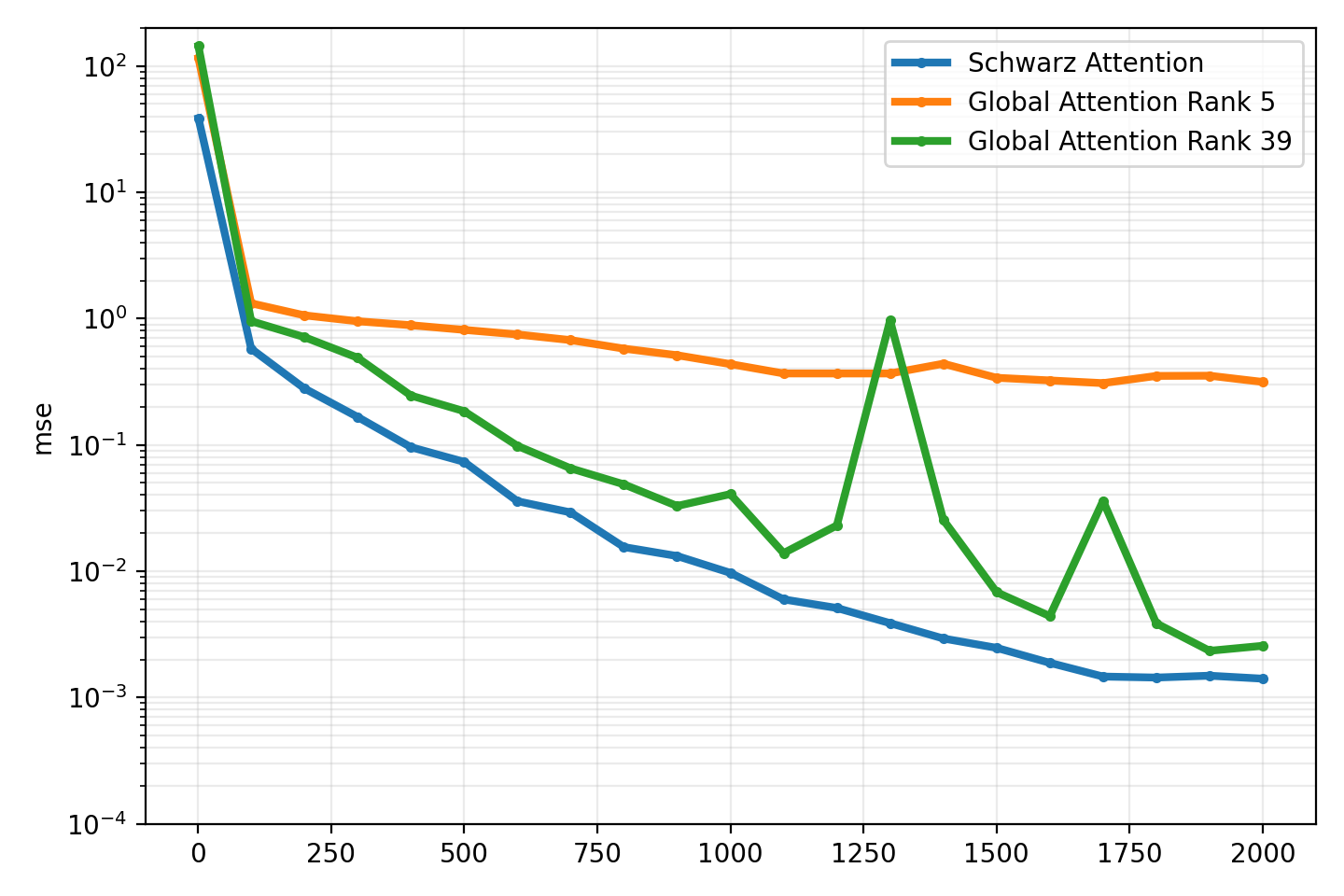}
    \caption{\(\eta=3\cdot10^{-4}\)}
  \end{subfigure}

  \vspace{0.5em}

  \begin{subfigure}{0.40\textwidth}
    \centering
    \includegraphics[width=\linewidth]{figures/mse_plots_by_lr_schwarz_r5_r39_korrigiert/mse__lr_0p001__cases_schwarz_global_r5_global_r39.png}
    \caption{\(\eta=10^{-3}\)}
  \end{subfigure}
  \hfill
  \begin{subfigure}{0.40\textwidth}
    \centering
    \includegraphics[width=\linewidth]{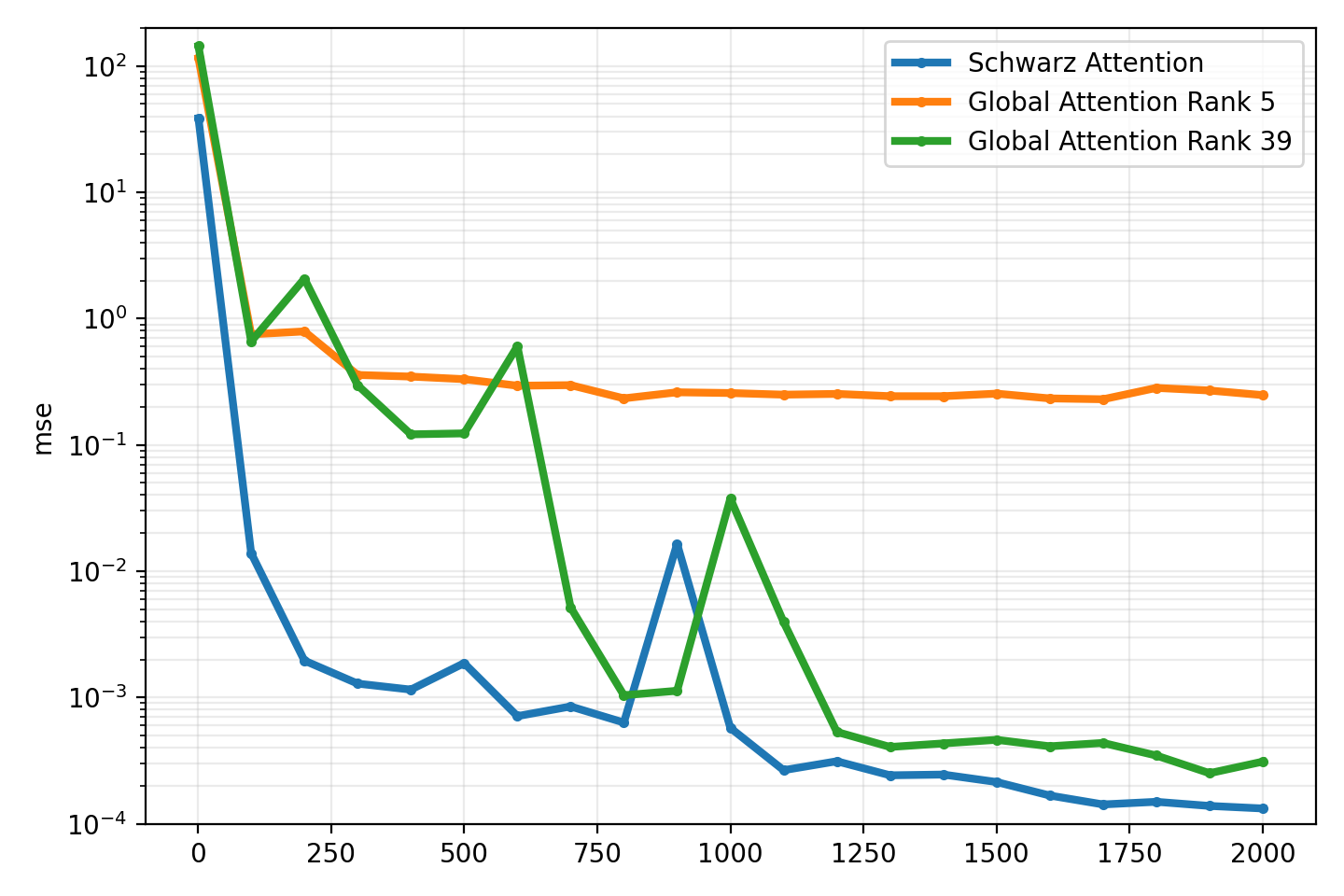}
    \caption{\(\eta=3\cdot10^{-3}\)}
  \end{subfigure}

  \vspace{0.5em}

  \begin{subfigure}{0.40\textwidth}
    \centering
    \includegraphics[width=\linewidth]{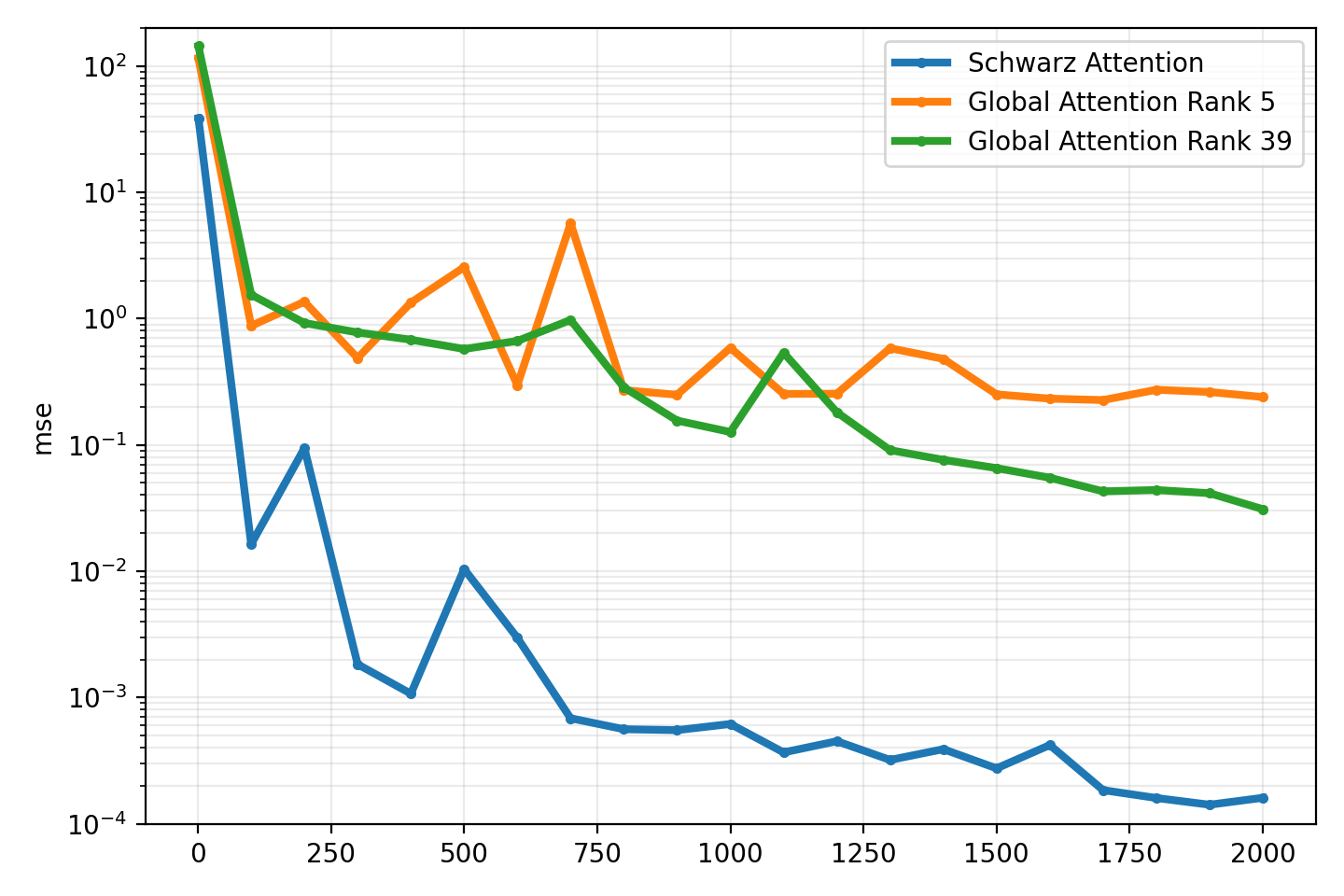}
    \caption{\(\eta=10^{-2}\)}
  \end{subfigure}
  \hfill
  \begin{subfigure}{0.40\textwidth}
    \centering
    \includegraphics[width=\linewidth]{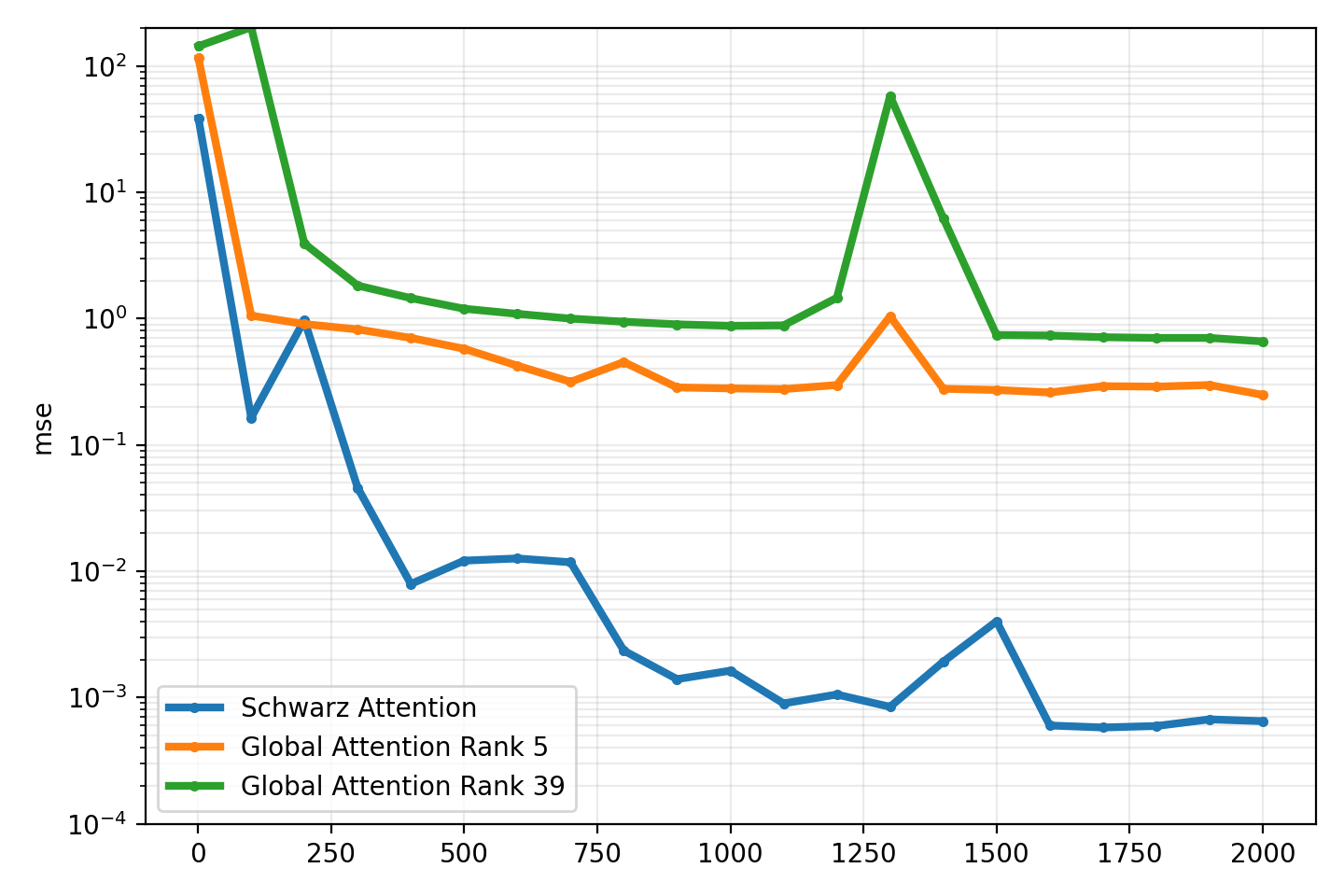}
    \caption{\(\eta=3\cdot10^{-2}\)}
  \end{subfigure}

%

  \caption{Learning-rate sweep for global low-rank attention and hierarchical domain-decomposition attention for $n=256$ and $N=8$.
  Each panel shows the weighted training wMSE 
  on the current training batch.
  The same training right-hand-side sequence is used for both models.
  The Schwarz attention uses local rank $4$ and coarse rank $7$, thus the rank is bounded by $8\times 4+7=39$. The global attention uses rank $39$. The global attention with rank $5$ is roughly comparable to the Schwarz attention with respect to the number of parameters.
  }
  \label{fig:lr-sweep-training-curves}
  \label{fig:lr-sweep-model-comparison}
  \label{fig:lr-comp}
\end{figure}

Across the learning-rate sweep
$\eta\in\{10^{-4},3\cdot 10^{-4},10^{-3},3\cdot 10^{-3},10^{-2},
3\cdot 10^{-2}\}$, Schwarz attention consistently outperforms the parameter-matched rank-5 baseline and reaches substantially lower weighted MSE values; see Figure~\ref{fig:lr-comp}. This shows that the improvement is not merely a consequence of using more parameters.

The comparison with the rank-39 global model is more stringent, since this baseline has many more trainable parameters. Nevertheless, Schwarz attention reaches lower final errors for the learning rates considered and exhibits a more stable optimization behavior. The global rank-39 model can reduce the loss substantially compared with rank 5, but its curves show stronger sensitivity to the learning rate and more pronounced transient spikes. The advantage of Schwarz attention is most visible for moderate learning rates, in particular around $\eta=10^{-3}$ and $3\cdot 10^{-3}$, where it combines rapid convergence with low final error.

The parameter-matched global rank-5 baseline behaves qualitatively differently from both Schwarz attention and the global rank-39 model. It decreases the loss initially, but then stagnates at a substantially higher error level and does not reach the low-MSE regime attained by the two higher-rank models. Thus, the rank-5 comparison shows the limitation of a purely global low-rank approximation under a comparable parameter budget. In contrast, Schwarz attention reaches the accuracy regime of the much larger rank-39 global model, and in the relevant learning-rate range even improves upon it.

\subsection{Additional Fixed Visualization Examples}
\label{app:visual}
Figures~\ref{fig:visual-schwarz}, \ref{fig:visual-global-r39}, and \ref{fig:visual-global-r5}
show the learned approximations for $12$ sample right-hand sides for the overlapping Schwarz attention, global attention with rank 39, and global attention with rank 5.
These right-hand sides are sampled from the same prescribed distribution as the
training right-hand sides, but using a different random seed, and are kept fixed
in order to make the visual comparison reproducible.


The purpose of these plots is not to assess out-of-distribution
generalization. Instead, they provide a diagnostic view of the learned operator
on representative samples from the training distribution, analogous to the
weighted MSE values shown in the training curves in Figure~\ref{fig:lr-sweep-training-curves}.

The sample solutions in
Figure~\ref{fig:visual-schwarz} -- \ref{fig:visual-global-r5}
illustrate the quantitative findings. Schwarz attention and global attention with rank $39$ both give accurate reconstructions on the displayed validation examples; the predicted and exact solution curves are often nearly
indistinguishable. By contrast, the parameter-matched global rank-5 model shows visible errors, particularly for more oscillatory right-hand sides. This confirms that the rank-5 global model lacks the expressive power needed to reach the low-error regime, whereas the Schwarz construction achieves accuracy comparable to the much larger rank-39 global model with far fewer parameters.

\begin{figure}[htbp]
  \centering
  \includegraphics[width=.99\linewidth]{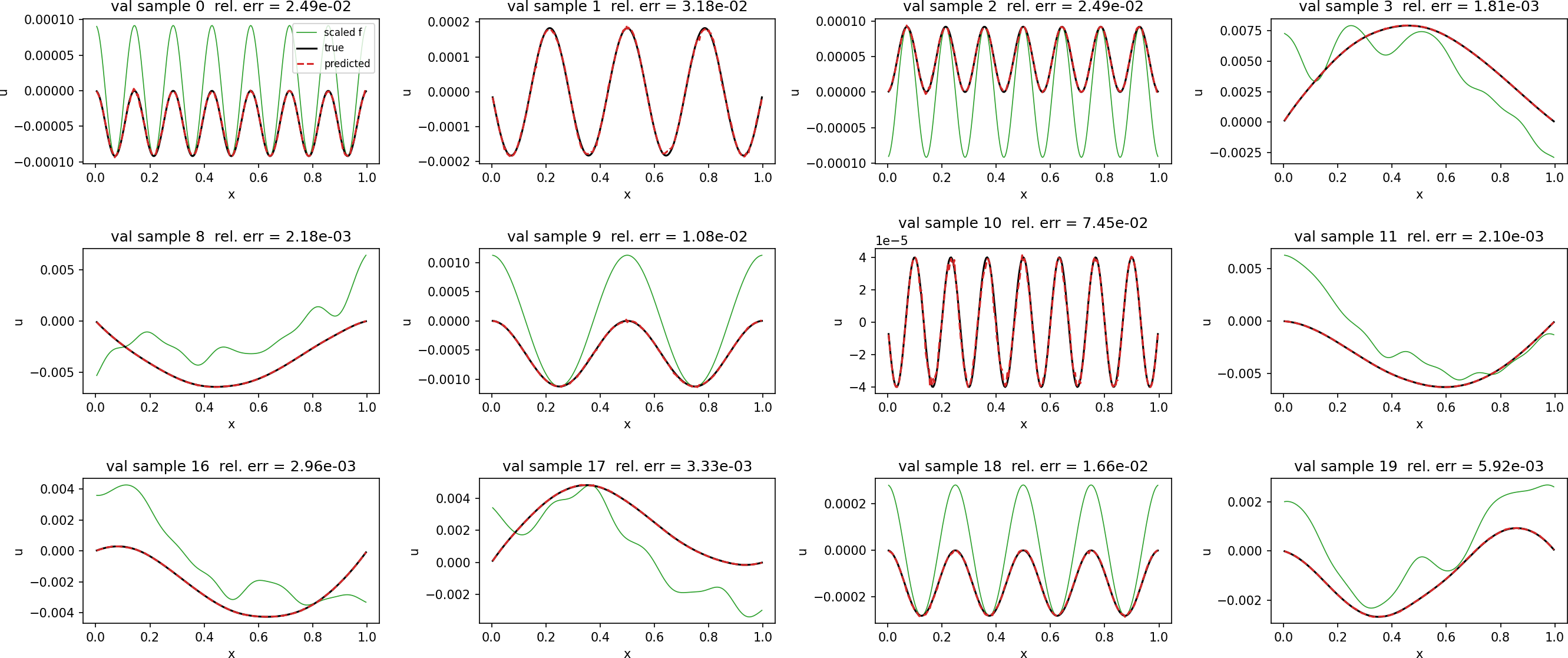}
  \caption{
    Sample solutions for overlapping Schwarz attention after 2000
    iterations with learning rate $\eta=0.001$.
    The exact solution is plotted in black and the model prediction as
    a dashed red line. The right-hand side is shown in green and was
    rescaled to fit the plotting range.
  }
  \label{fig:visual-schwarz}
\end{figure}

\begin{figure}[htbp]
  \centering
  \includegraphics[width=.99\linewidth]{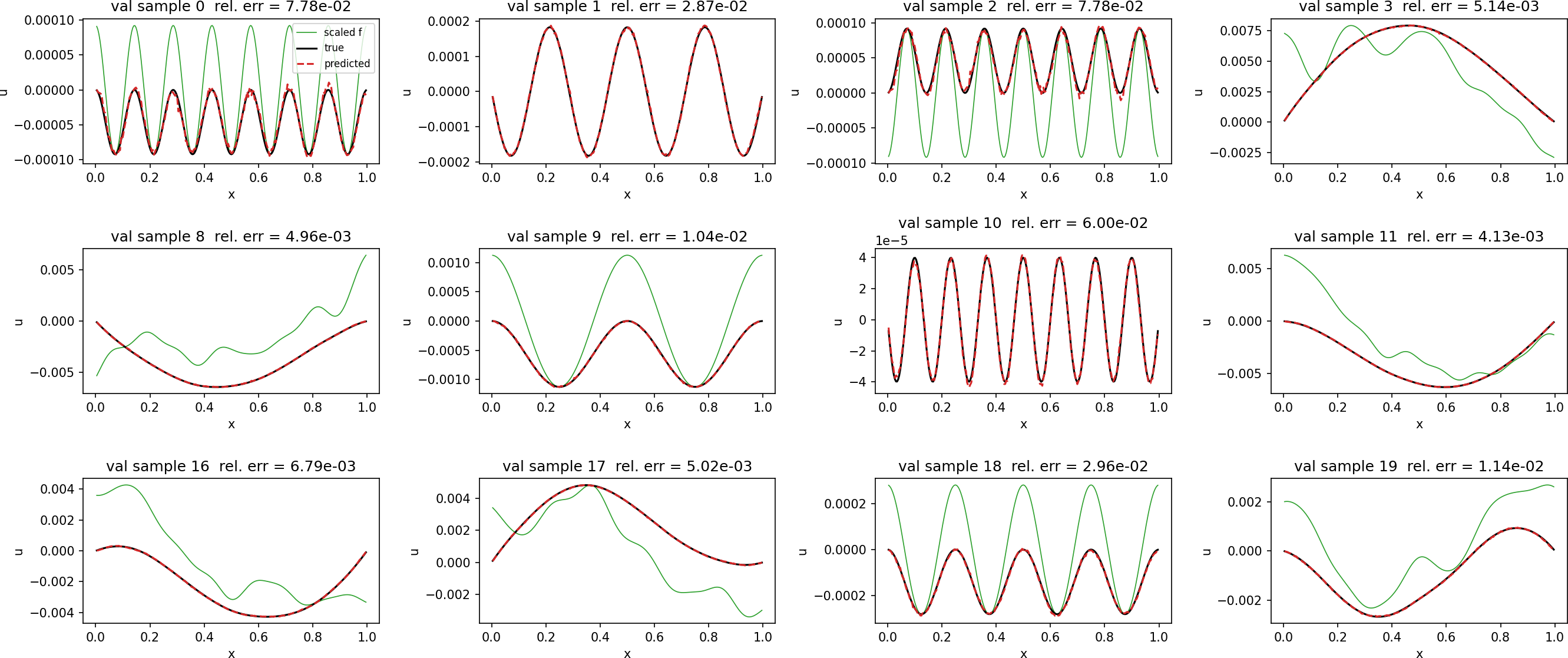}
  \caption{
    Sample solutions for global attention with rank $r=39$ after 2000
    iterations with learning rate $\eta=0.001$.
    The exact solution is plotted in black and the model prediction as
    a dashed red line. The right-hand side is shown in green and was
    rescaled to fit the plotting range.
  }
  \label{fig:visual-global-r39}
\end{figure}

\begin{figure}[htbp]
  \centering
  \includegraphics[width=.99\linewidth]{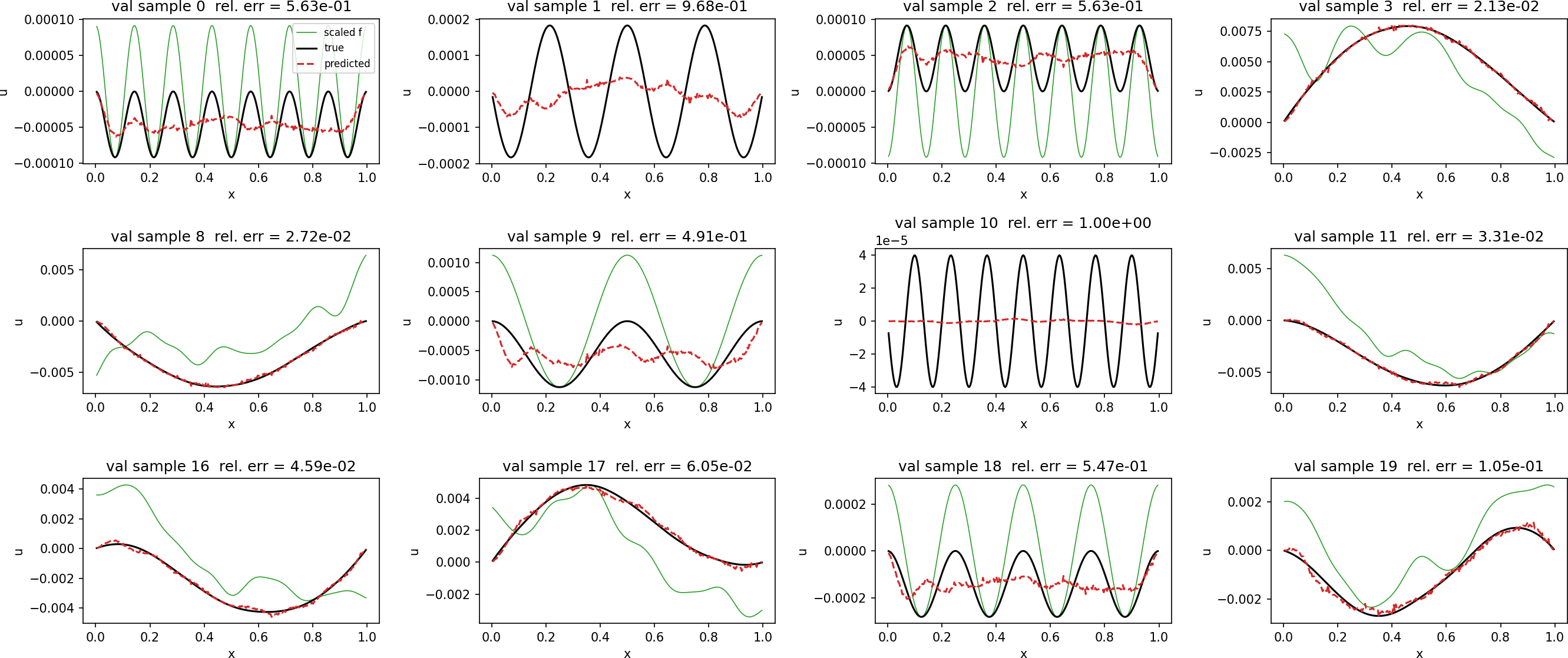}
  \caption{
    Sample solutions for global attention with rank $r=5$ after 2000
    iterations with learning rate $\eta=0.001$.
    The exact solution is plotted in black and the model prediction as
    a dashed red line. The right-hand side is shown in green and was
    rescaled to fit the plotting range.
  }
  \label{fig:visual-global-r5}
\end{figure}

\FloatBarrier
\subsection{Training Curves for Increasing Numbers of Local Attention Blocks}
\label{app:weak}
\begin{figure}[h!]
  \begin{center}
    \begin{subfigure}{0.48\linewidth}
      \centering
      \includegraphics[width=\linewidth]{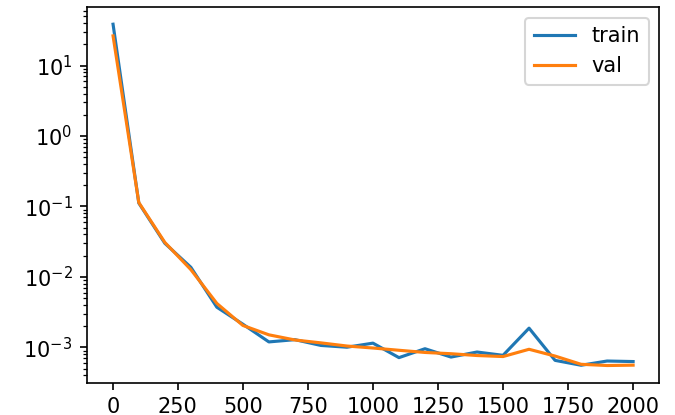}
      \caption{$n=256$, $N=8$}
    \end{subfigure}\hfill
    \begin{subfigure}{0.48\linewidth}
      \centering
      \includegraphics[width=\linewidth]{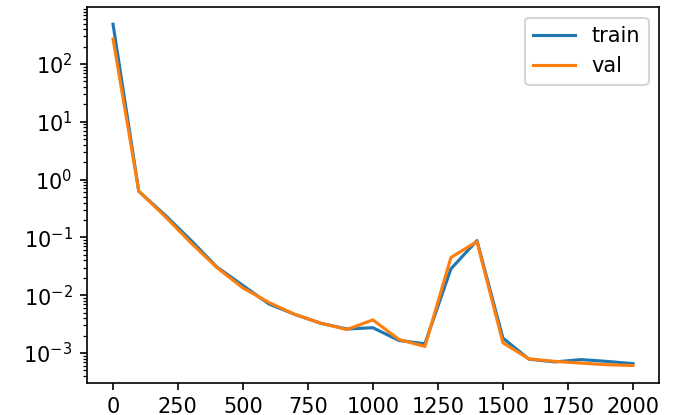}
      \caption{$n=512$, $N=16$}
    \end{subfigure}\\
    \begin{subfigure}{0.48\linewidth}
      \centering
      \includegraphics[width=\linewidth]{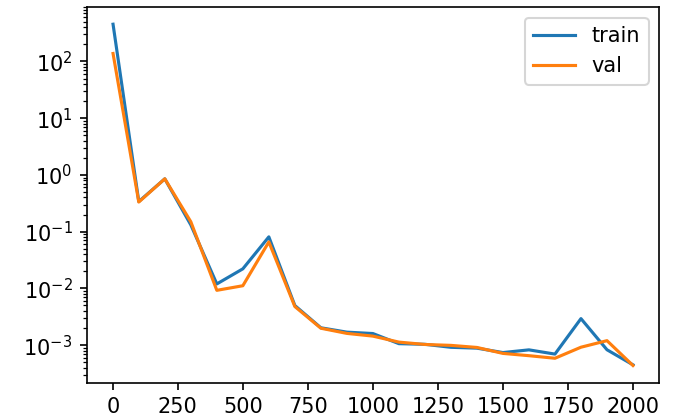}
      \caption{$n=1024$, $N=32$}
    \end{subfigure}\hfill
    \begin{subfigure}{0.48\linewidth}
      \centering
      \includegraphics[width=\linewidth]{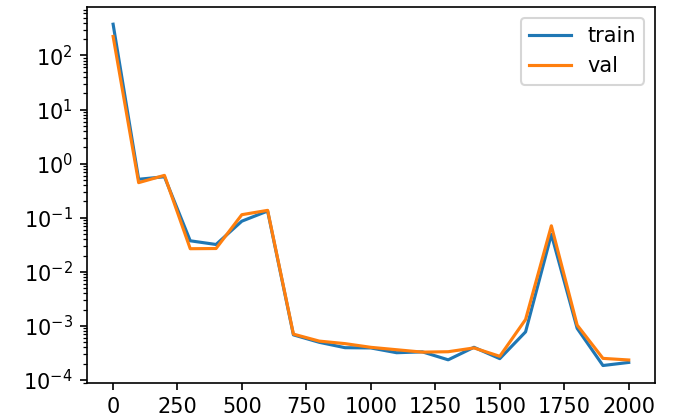}
      \caption{$n=2048$, $N=64$}
    \end{subfigure}\\
    \begin{subfigure}{0.48\linewidth}
      \centering
      \includegraphics[width=\linewidth]{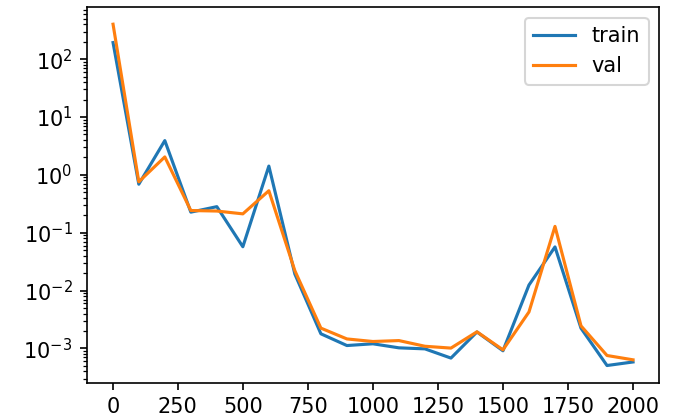}
      \caption{$n=4096$, $N=128$}
    \end{subfigure}\hfill
    \begin{subfigure}{0.48\linewidth}
      \centering
      \includegraphics[width=\linewidth]{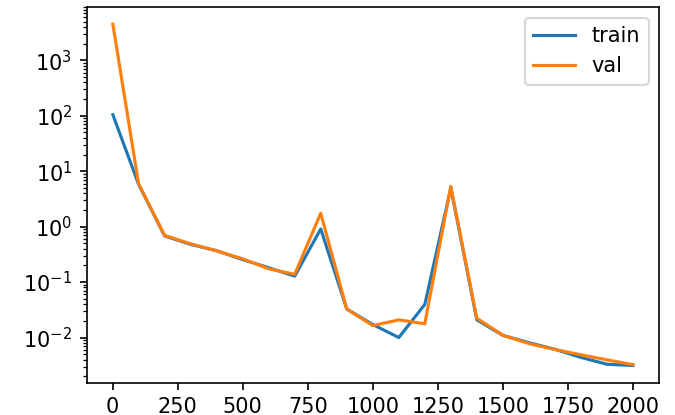}
      \caption{$n=8192$, $N=256$}
    \end{subfigure}\\
    \caption{Training and validation loss (wMSE) over 2000 training steps for an increasing problem size and number of local attention blocks.
}
\label{appfig:weak}
\end{center}
\end{figure}

In Figure~\ref{appfig:weak}, the final weighted MSE remains small over the full range $n=256,\ldots,8192$. Up to $n=4096$, the final errors stay in the range of about $10^{-4}$ to $10^{-3}$, and even for $n=8192$ the method still reaches a weighted MSE of order $10^{-3}$. This indicates that the two-level Schwarz attention construction remains trainable as the number of local blocks is increased substantially, without increasing the local rank.

\subsection{Proofs of Theorems~1 and~2}
\label{app:proofs}
\begin{proof}[Proof of Theorem~\ref{theo:one}]
The product \(QK^T\) can represent any matrix of rank at most \(r_g\).
Since \(A^{-1}\) is symmetric positive definite, its singular values are
\(\lambda_k^{-1}\). The result follows from the
Eckart--Young--Mirsky theorem.
\end{proof}

For the Schwarz-attention operator \(M_{\theta,\mathrm{Schwarz}}^{-1}\)
defined in \eqref{eq:dd-attention}, we use the already introduced local
and coarse blocks
\[
  G_0 = Q_0K_0^T,
  \qquad
  G_i = Q_iK_i^T,
  \qquad i=1,\ldots,N .
\]
Assume that
\[
  \operatorname{rank}(G_0)\leq r_0,
  \qquad
  \operatorname{rank}(G_i)\leq r_i,
  \qquad i=1,\ldots,N .
\]
Let \(n_i\) denote the number of degrees of freedom on the \(i\)-th
overlapping subdomain, and let \(\mathcal{S}_{\mathrm{Schwarz}}\) denote
the set of all Schwarz-attention operators of the form
\eqref{eq:dd-attention} satisfying these rank bounds.

\begin{proof}[Proof of Theorem~\ref{theo:two}]
The rank of the coarse contribution is bounded by
\[
  \operatorname{rank}(\Phi G_0\Phi^T)
  \leq
  \min\{r_0,\operatorname{rank}(\Phi)\}.
\]
Similarly,
\[
  \operatorname{rank}
  \left(
    R_i^T D_i^{1/2}G_iD_i^{1/2}R_i
  \right)
  \leq
  \min\{r_i,n_i\}.
\]
The rank estimate follows by subadditivity of the rank. Since the best
unstructured rank-\(R_{\mathrm{total}}\) approximation of \(A^{-1}\) has
Frobenius error
\[
  \left(
  \sum_{k=R_{\mathrm{total}}+1}^{n-1} \lambda_k^{-2}
  \right)^{1/2},
\]
no operator in \(\mathcal{S}_{\mathrm{Schwarz}}\), whose rank is bounded
by \(R_{\mathrm{total}}\), can have a smaller Frobenius error.
\end{proof}

\subsection{Visualization of the Learned Operators}
The learned operator structure can be inspected by visualizing the global operator, the local domain-decomposition contribution, the coarse contribution, and the full learned operator; see Figure~\ref{fig:operator-matrices}. Figure~\ref{fig:operator-matrices} also visualizes the interface hat functions in the lower rightmost panel.
In Figure~\ref{fig:local-blocks} the learned attention blocks are visualized.
Figure~\ref{fig:phi-matrix} visualizes the $\Phi^T$-matrix, which collects the hat functions as its rows.

\begin{figure}[htb]
  \centering
  \includegraphics[width=\textwidth]{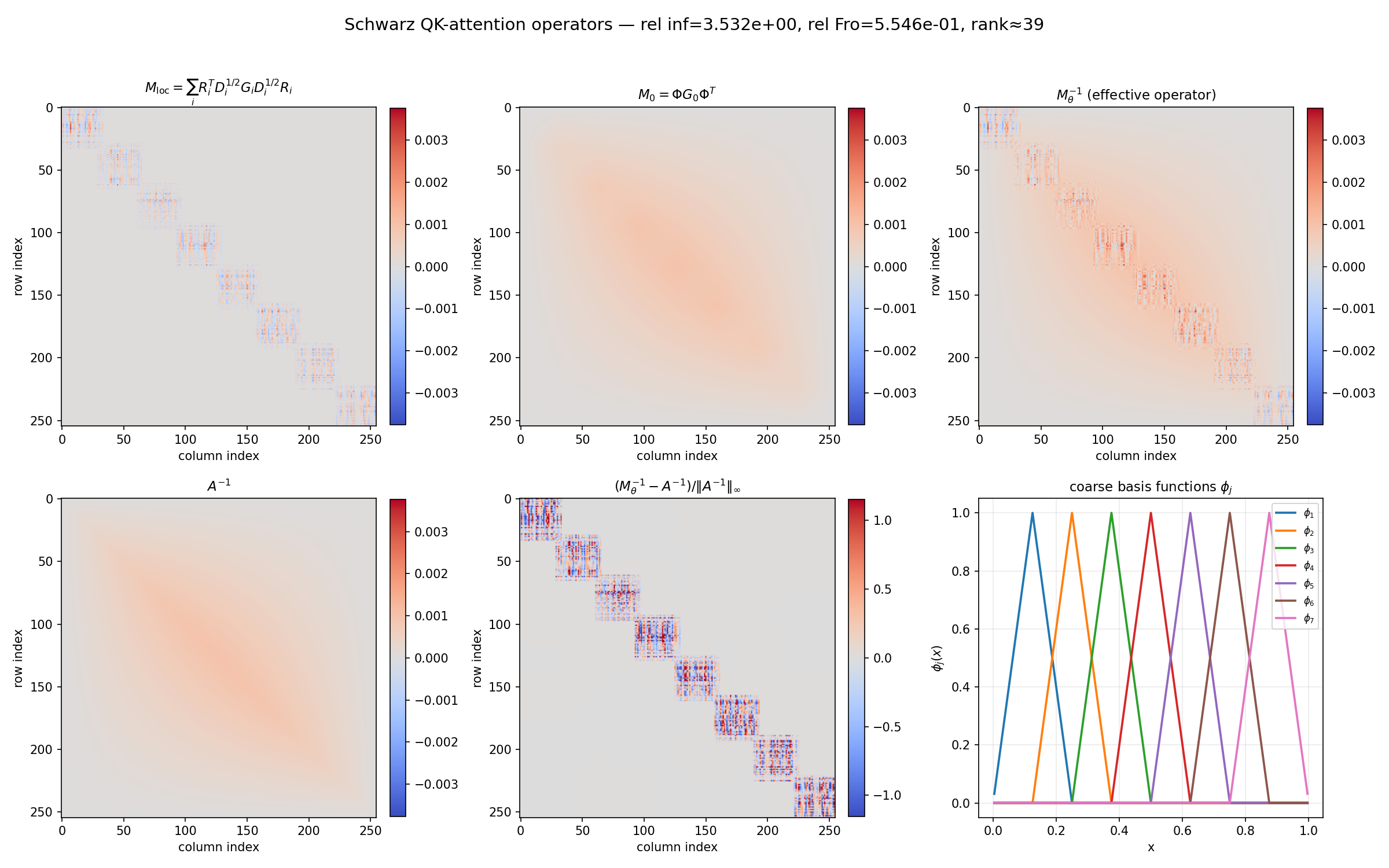}
  \caption{Visualization of the
 learned hierarchical domain-decomposition
attention operator.
Top row, from left to right: sum of the local overlapping
Schwarz contribution
$M_{\mathrm{loc}}=\sum_i R_i^T D_i^{1/2}G_iD_i^{1/2}R_i$,
the coarse contribution $M_0=\Phi G_0\Phi^T$, and the full learned
operator $M_\theta^{-1}=M_{\mathrm{loc}}+M_0$.
Bottom row, from left to right: the exact inverse $A^{-1}$, the
normalized error
$(M_\theta^{-1}-A^{-1})/\|A^{-1}\|_\infty$, and the interface-hat
coarse basis functions $\phi_j$ used in the interpolation matrix $\Phi$.
The local contribution is block-local due to the subdomain restrictions,
whereas the coarse contribution introduces global coupling; hence the
assembled two-level operator is generally dense.
  %
  }
  \label{fig:operator-matrices}
\end{figure}

\begin{figure}[htb]
  \centering
  \includegraphics[width=\textwidth]{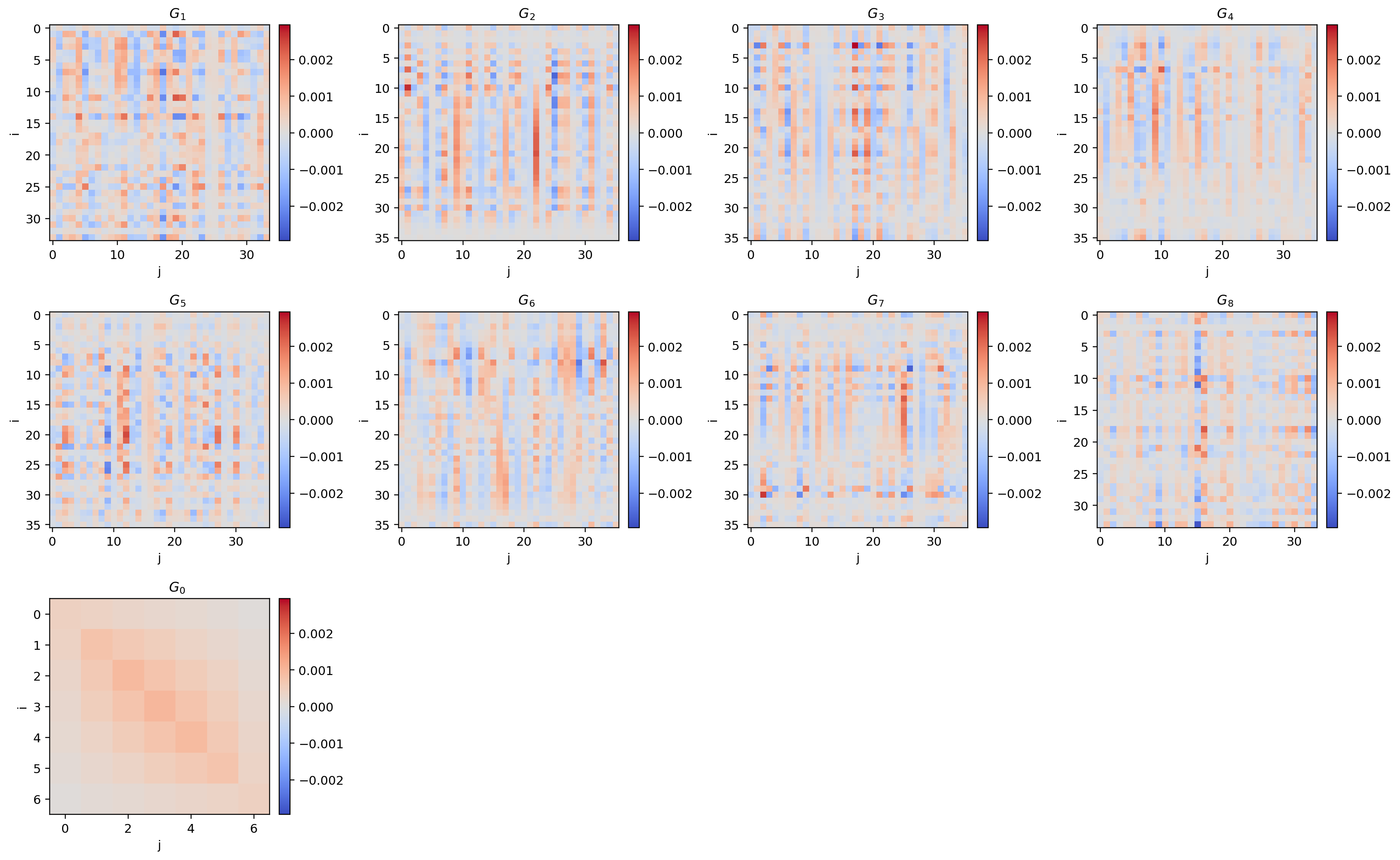}
  \caption{
  Example of learned local low-rank attention blocks \({\color{black}G_i=}Q_iK_i^T\) and coarse block \({\color{black}G_0=}Q_0K_0^T\) for $8$ subdomains and a coarse attention of size $7$.}
  \label{fig:local-blocks}
\end{figure}

\begin{figure}[htb]
  \centering
  \includegraphics[width=0.75\textwidth]{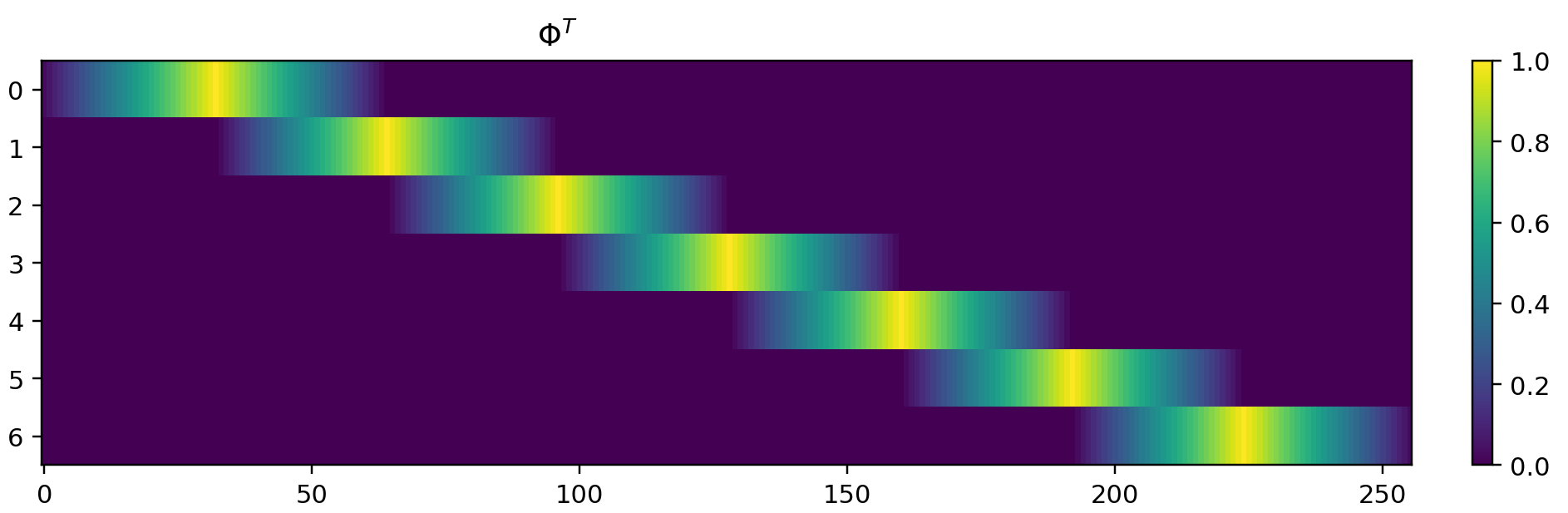}
  \caption{
  Interface-hat coarse basis used in the restriction matrix $\Phi^T$. For the decomposition into eight disjoint index sets, one hat function is associated with each pair of neighboring index sets, giving seven coarse basis functions.
These functions span the coarse space on which the coarse attention block $G_0=Q_0K_0^T$ acts before being lifted to the fine grid as $\Phi G_0\Phi^T$.
  }
  \label{fig:phi-matrix}
\end{figure}

\end{document}